\documentclass[12pt]{article}

\usepackage{scicite}

\usepackage{times}

\newenvironment{sciabstract}{%
\begin{quote} \bf}
{\end{quote}}

\usepackage{booktabs} 
\usepackage{multirow}
\usepackage{amssymb}
\usepackage{amsfonts}
\usepackage{amsthm}
\usepackage{amsmath}
\usepackage{bm}
\usepackage{subcaption}
\usepackage{graphicx}
\usepackage{wrapfig}
\usepackage[linesnumbered,ruled,vlined,noend]{algorithm2e} 
\usepackage[hyphens]{url}
\usepackage{xcolor}
\usepackage{enumitem}

\SetAlFnt{\small}
\SetAlCapFnt{\small}
\SetAlCapNameFnt{\small}
\SetAlCapHSkip{0pt}
\IncMargin{-\parindent}

\newtheorem{definition}{Definition}
\newtheorem{theorem}{Theorem}
\newtheorem{lemma}{Lemma}

\newtheorem{corollary}{Corollary}

\definecolor{ao(english)}{rgb}{0.0, 0.5, 0.0}
\definecolor{ckcolor}{rgb}{0.5, 0.0, 0.5}
\definecolor{jpdcolor}{rgb}{0.5, 0.5, 0.0}
\definecolor{kacolor}{rgb}{0.5, 0.5, 0.5}

\def\comments{1}

\if\comments1

\newcommand{\DCM}[1]{{\bfseries \color{ao(english)}DCM: #1}}

\newcommand{\JPD}[1]{{\bfseries \color{jpdcolor}JPD: #1}}
\newcommand{\KA}[1]{{\bfseries \color{kacolor}KA: #1}}
\else
\newcommand{\DCM}[1]{}
\newcommand{\CK}[1]{}
\newcommand{\JPD}[1]{}
\newcommand{\KA}[1]{}
\newcommand{\todo}[1]{}
\fi

\newcommand{\fair}{proportional}
\newcommand{\fairness}{proportionality}
\newcommand{\Fair}{Proportional}

\newcommand{\MAX}{\texttt{Max}}  
\newcommand{\OPT}[1]{\texttt{OPT}(#1)}

\newcommand{\RAND}{\texttt{Rand}}
\newcommand{\mT}{\mathcal{T}}

\title{\vspace{-1in}
Matching Algorithms for Blood Donation}

\author{%
Duncan C.\ McElfresh,$^{1\ast}$ Christian Kroer,$^{2}$ Sergey Pupyrev,$^{3}$ Eric Sodomka,$^{3}$\\
Karthik Sankararaman,$^{3}$ Zack Chauvin,$^{3}$ Neil Dexter,$^{3}$ John P.\ Dickerson$^{1}$\\
\\
\normalsize{$^{1}$University of Maryland, College Park, MD 20742, USA.}\\
\normalsize{$^{2}$Columbia University, New York, NY 10027, USA.}\\
\normalsize{$^{3}$Facebook, Menlo Park, CA 94025, USA.}
}

\date{}

\begin{document}


\baselineskip24pt

\maketitle 

\begin{sciabstract}
    Global demand for donated blood far exceeds supply, and unmet need is greatest in low- and middle-income countries; experts suggest that large-scale coordination is necessary to alleviate demand.
    Using the Facebook Blood Donation tool, we conduct the first large-scale algorithmic matching of blood donors with donation opportunities.
    While measuring actual donation rates remains a challenge, we measure \emph{donor action} (e.g., making a donation appointment) as a proxy for actual donation.
    We develop automated policies for matching donors with donation opportunities, based on an online matching model. We provide theoretical guarantees for these policies, both regarding the number of expected donations and the equitable treatment of blood recipients.
    In simulations, a simple matching strategy increases the number of donations by $5$-$10\%$; a pilot experiment with real donors shows a $5\%$ relative increase in donor action rate (from $3.7\%$ to $3.9\%$).
    When scaled to the global Blood Donation tool user base, this corresponds to an increase of around one hundred thousand users taking action toward donation.
    Further, observing donor action on a social network can shed light onto donor behavior and response to incentives.
    %
    Our initial findings align with several observations made in the medical and social science literature regarding donor behavior.
\end{sciabstract}




%


\section{Introduction}\label{sec:intro}

Blood is a scarce resource; its donation saves the lives of those in need.  
Countries approach blood donation in different ways, running the gamut from privately-run to state-run programs, with or without monetary compensation, and with varying degrees of public campaigns for action.\footnote{Some examples follow. China maintains state control of its donation centers, which take a mix of captive-, quota-, and voluntary-based donations~\cite{Guan18:Voluntary}.
The US mixes state- and private-run donation that is primarily sourced via voluntary donations~\cite{Osorio15:Structured}.  
Brazil has seen a recent shift from remunerated to non-remunerated (aka voluntary) donation at its initially state-run, and now Federally-run, centers~\cite{Carneiro10:Demographic}.}  
As such, blood donation rates differ across different countries; for example, approximately 3.2\%, 1.5\%, 0.8\%, and 0.5\% of the population donates in high-, upper-middle-, lower-middle-, and low-income countries, with varying rates of voluntary versus paid donors~\cite{WHO17:Blood}.  
Yet demand for blood still far exceeds supply, and unmet need is greatest in low- and middle-income countries~\cite{roberts2019global}.
Thus, experts suggest that the blood supply chain---collection, testing, processing, storage, and distribution---be managed at a national level~\cite{WHO17:Blood,roberts2019global}.

Optimization-based approaches to blood supply chain management have a rich history in the operations research and healthcare management literature. 
\cite{Osorio15:Structured} reviews over 100 publications in this space since 1963.  
The supply chain is roughly split into collection, testing \& processing, storage \& inventory, and distribution \& transfusion~\cite{Osorio17:Simulation}.  
Substantial research effort has gone into each of those segments~\cite{Katsaliaki07:Using,Zahiri17:Blood,Dillon17:Two,el2018robust,prastacos1980pbds}.
Yet, we note that most optimization-based research in the initial \emph{collection} stage of the blood supply chain has focused on \emph{prediction} of blood supply (e.g., during a crisis).  
In this work, we instead focus on the \emph{creation} and \emph{coordination} of blood supply via automated social prompts, subject to the expressed preferences and constraints of potential donors and the overall donation system.
That is, we focus on the \emph{donor recruitment} stage of the blood supply chain (see Figure~\ref{fig:flowchart}).

\begin{figure}[ht!]
    \centering
    \includegraphics[width=0.85\textwidth]{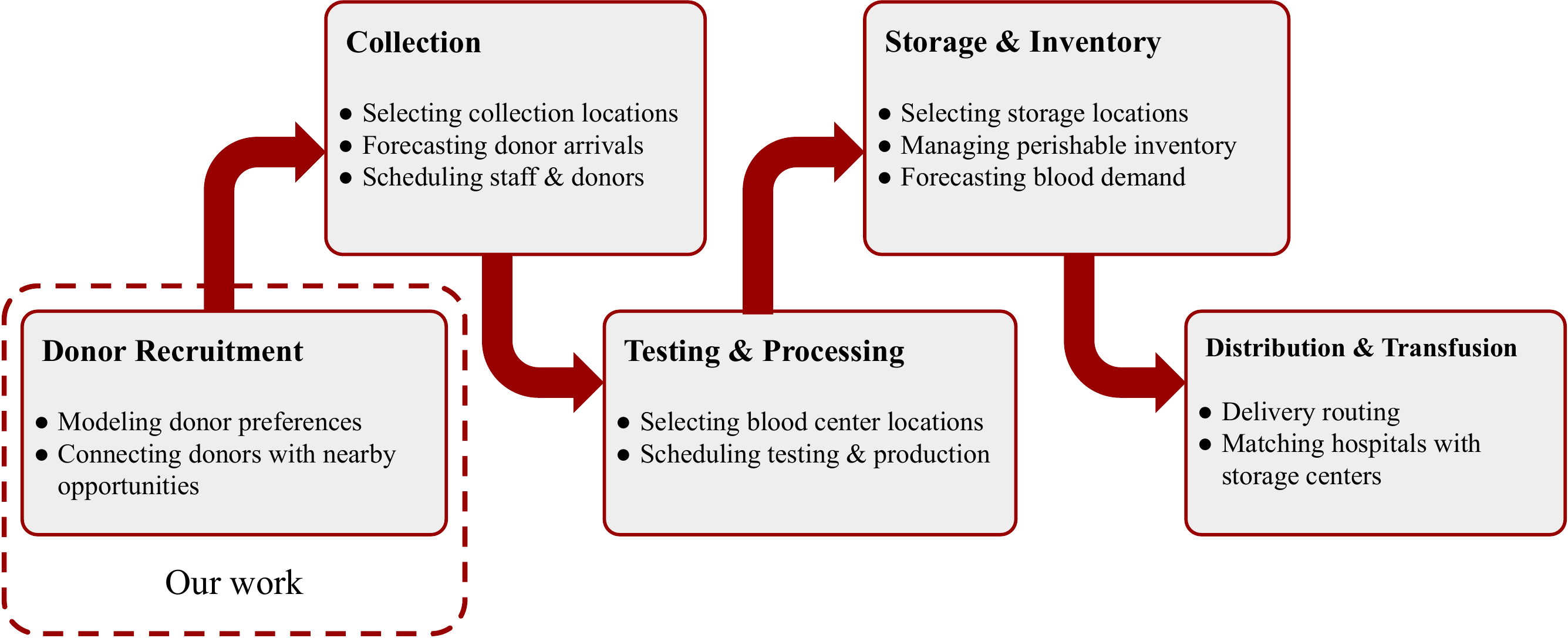}
    \caption{Stages of the blood supply chain. Our work---donor recruitment---precedes the four stages of the blood supply chain as described in \cite{Osorio15:Structured}.}
    \label{fig:flowchart}
\end{figure}

Donor recruitment has also been a topic of study for decades.
Factors like social pressure~\cite{sojka2008blood}, empathetic messaging~\cite{reich2006randomized}, and non-monetary incentives~\cite{chell2018systematic} can increase donation rates.
Negative past experiences, and real or perceived barriers to donation, can also impede donation rates~\cite{van2014predicting,godin2005factors,craig2017waiting}.
Most importantly, this body of work suggests that \emph{different donors are motivated by different factors}.
In other words, personalized recruitment strategies---which respect diverse donor motivations, preferences, and perceived barriers to donation---should be more effective than a uniform recruitment strategy.

Our work leverages the widespread use of web-based applications (apps) and social media platforms, which already play a substantial role in blood donor recruitment.
The American Red Cross, which provides about 40\% of transfused blood in the United States,\footnote{\url{https://www.redcrossblood.org/donate-blood/how-to-donate/how-blood-donations-help/blood-needs-blood-supply.html}} recently launched an app to connect blood donors with donation opportunities.\footnote{\url{https://www.redcrossblood.org/blood-donor-app.html}}
A review by~\cite{ouhbi2015free} identifies 169 free mobile apps for blood donation; though many of these apps have usability and privacy issues that may prevent widespread use.
In a survey of donors at a German hospital, \cite{sumnig2018role} finds that social media platforms Jodel and Facebook are a major motivation for donation---especially for first-time donors.
Similar studies find that WhatsApp and Twitter help promote donation in Saudi Arabia~\cite{alanzi2019use} and India~\cite{abbasi2018saving}.
%

\begin{figure}
    \centering
    \begin{subfigure}[t]{.39\textwidth}
  \centering
    \includegraphics[width=\linewidth]{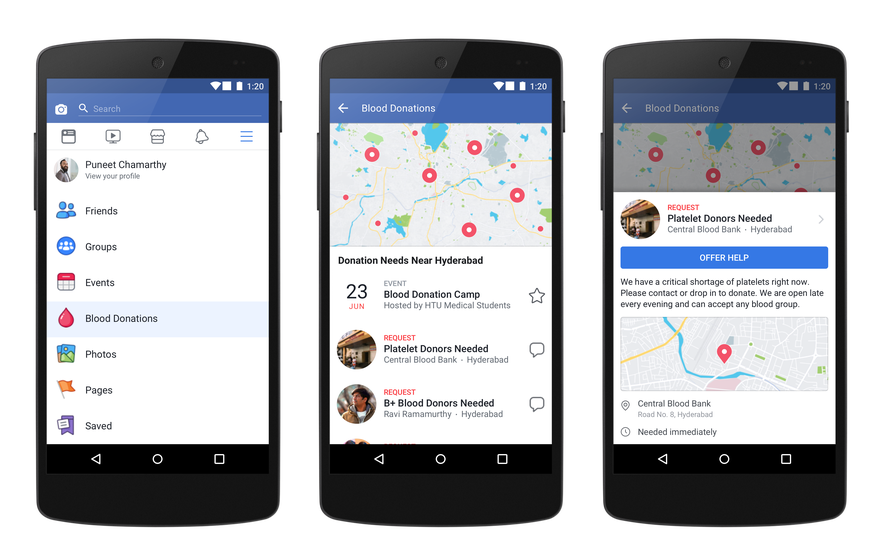}
  \caption{}
  \label{fig:fb-blood-donation}
\end{subfigure}%
\begin{subfigure}[t]{.59\textwidth}
  \centering
  \includegraphics[width=\linewidth]{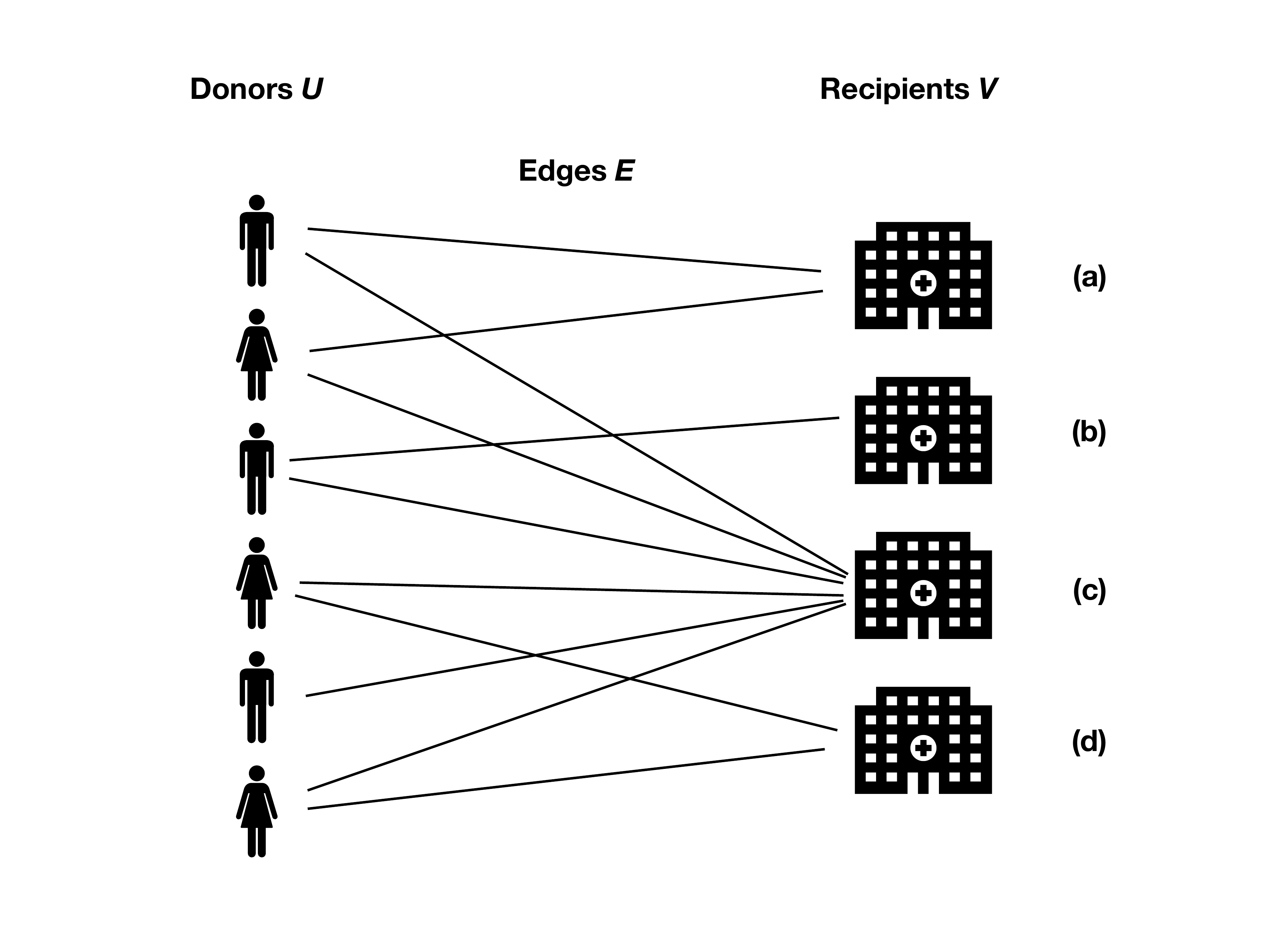}
  \caption{}
  \label{fig:matching}
\end{subfigure}
    \caption{(a) The Facebook Blood Donation tool interface, where users can search for donation opportunities, and opt in to receive notifications about nearby opportunities as they arise. (Source: \protect\url{https://about.fb.com/news/2018/06/making-it-easier-to-donate-blood}.) (b) an example matching graph, with donors (Facebook users who opt in to receive notifications about nearby opportunities), recipients (e.g., hospitals and blood banks), and edges (potential notifications that can be sent to donors).}
\end{figure}

Herein we propose a personalized donor recruitment strategy using the recently developed Facebook Blood Donation tool,\footnote{\texttt{https://socialgood.fb.com/health/blood-donations/}} which connects millions of potential blood donors with opportunities to donate, in several countries around the world. 
Users of this tool can opt in to receive notifications about nearby donation opportunities. 
Our strategy aims to notify donors about opportunities they are \emph{more likely} to take action on.
We frame this notification scenario as an online bipartite matching problem~\cite{Karp90:Optimal}---a well-studied paradigm which has been applied to a variety of settings including online advertising~\cite{Mehta07:AdWords} and rideshare services~\cite{Dickerson18:Allocation,lowalekar2018online,wang2018stable}. 
We demonstrate, both in computational simulations and in a real A/B test, that even a simple matching policy can substantially increase the likelihood of donor action. 
%

\section{Online Platform: the Facebook Blood Donation Tool}\label{sec:materials-methods}

The advent of global social networks offers a unique opportunity to recruit and coordinate massive numbers of donors, in order to meet a large and unpredictable demand for donor blood. 
The Facebook Blood Donation Tool aims to seize this opportunity---leveraging the widespread use of its online platform to connect blood donors with nearby recipients (see Figure~\ref{fig:fb-blood-donation}).
Donors can also opt in to receive \emph{notifications} about nearby donation opportunities. 
This tool is available in several countries around the world;\footnote{As of February 2021, the Blood Donation Tool has been approved in Bangladesh, Brazil, Burkina Faso, Chad, Cote d’Ivoire, Egypt, England, Guinea, Hong Kong, India, Kenya, Mali, Mexico, Mongolia, Namibia, Netherlands, Niger, Northern Ireland, Pakistan, Peru, Rwanda, Senegal, South Africa, the United States, Taiwan, Wales and Zimbabwe (see \texttt{\url{https://socialimpact.facebook.com/health/blood-donations/}}).} as of December 2020, more than 85 million people have registered with this tool.\footnote{\texttt{\url{https://socialimpact.facebook.com/health/blood-donations/}}.}

In this paper we focus on a small but important feature of the Blood Donation Tool: automatic donor notifications. 
Our primary goal is to \emph{increase the number of blood donations around the world} by carefully selecting \emph{which opportunity} to notify each donor about, and \emph{when} to notify them. 
We frame this question of donor notifications as an \emph{online matching problem}.
One might ask whether such a complicated approach is warranted in this setting---perhaps it does not matter how and when donors are notified. 
To better motivate our approach, we first answer the question: how can we tell whether a Facebook user donates blood after we notify them?

\subsection{Measuring Donation: Meaningful Action.}\label{sec:measuring}

To design notifications that effectively encourage blood donation, it is necessary to know \emph{when} donations occur.
However social networking platforms like Facebook cannot directly observe a user's action outside the platform. 
As a proxy, we instead observe when a donor takes \emph{meaningful action} toward donation after being notified. 
In our context, Meaningful Actions (MA) include user behaviors such as creating a reminder to donate, or calling a blood bank; note that these actions are only observed if taken within the Facebook platform.

It is beyond the scope of this study to validate MA as a proxy for actual donation, however initial results indicate that MA is a reliable indicator.
For example, a 2018 Facebook study with its partner donation sites in India and Brazil found that 20\% of donors said that Facebook influenced their decision to donate blood.\footnote{Ibid.}
%
In the remainder of this paper, we focus on increasing the number of donor MAs as a proxy for increasing the number of donations. 
Our goal is to design a notification policy that chooses both (a) \emph{when} to notify a donor, and (b) \emph{which donation opportunity} to notify them about. 
The next step in designing this policy is to understand \emph{which}
notifications are likely to prompt donor MA. We begin with some high-level observations.

As an initial analysis we consider all notifications sent to donors using the Facebook Blood Donation tool over a one-month period.\footnote{Hundreds of millions of notifications.} Below we describe some high-level observations; we leave a deeper analysis to future work.
\begin{enumerate}
    \item \textbf{Users rarely take meaningful action in response to notifications:} between 3\% and 4\% of all notifications lead to meaningful action.
    \item \textbf{More-engaged donors are more likely to take meaningful action:} Donors who tend to use Facebook every day are about 43\% more likely to take meaningful action in response to a notification than those who use Facebook about once per week.
    \item \textbf{New users are more likely to take action:} donors who joined Facebook within the last year are about 35\% more likely to take action in response to a notification that those who have been users for longer.
    \item \textbf{Older donors are more likely to take action:} donors over 30 years old are about 22\% more likely to take action in response to a notification than donors under 30.
    \item \textbf{Donors are more likely to take action if they are notified about a nearby opportunity:} Donors who are notified about opportunities less than 3km away are 20\% more likely to take action than those who are notified about further-away opportunities.
    \item \textbf{Donors are more likely to take action if they haven't been notified recently:} Donors who haven't been notified about a donation opportunity in the past 60 days are about 12\% more likely to take action in response to a notification than those who have been notified in the past 60 days.
\end{enumerate}

We emphasize that several of these observations have been reflected in prior studies:
(1) reflects the observation of \cite{sojka2008blood} and \cite{sumnig2018role} that social pressure and influence from family or friends can increase donation rates. 
(5) reflects the finding of \cite{van2014predicting} and \cite{godin2005factors} that logistical barriers to donation can impede donation rates.
(6) reflects the finding of \cite{yuan2016blood} that blood donors can be burdened by receiving too many notifications.

The likelihood of donor MA varies significantly across several features of both the blood donor (e.g., when they were last notified) and donation opportunity (e.g., location). 
To better understand these dependencies we train a predictive model for estimating likelihood of donor MA, using all available data from prior notifications. 
This model is used in both our offline and online experiments.

\subsection{Machine Learning Model for Donor Action}\label{sec:edgeweights}

To develop a machine learning (ML) model of donor action, we use all prior notifications sent by the Facebook Blood Donation tool. 
This model takes an individual notification as input, and predicts the \emph{probability} that the donor will take action. 
Each notification is represented by a set of \emph{features} of both the donor and the donation opportunity (i.e., the independent variables); the dependent variable is binary (i.e., whether or not the donor took MA).
Before being deployed, this ML model and application passed through Facebook's internal review process to protect user privacy.

Prior to training this model, we use industry-standard feature selection techniques to identify the most important features for predicting donor MA; these features are (in decreasing order of importance, with importance percentage in parenthesis): (1) whether the donor recently took meaningful action ($18\%$), (2) donor age ($8.5\%$), (3) donor city ($7.5\%$), (4) the number of Facebook friends the donor has ($7.3\%$), (5) the distance between donor and recipient ($6.8\%$).
Other relevant features include the number of local donors ($6.5\%$), number of times a donor has viewed the hub in the last $30$ days, and the number of days since the donor's last notification.

\begin{figure}
    \centering
    \includegraphics[width=0.4\textwidth]{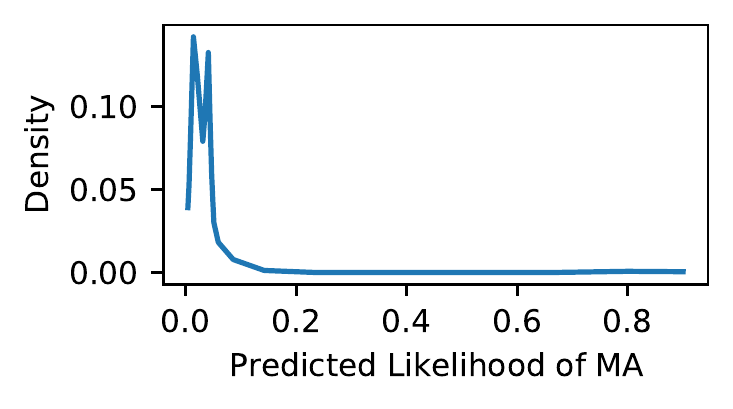}
    \caption{Density of estimated likelihood of MA, for all notifications in the training data.}
    \label{fig:predictionScores}
\end{figure}

Using the selected features, we train a gradient boosted decision tree (GBDT) model. %
We use standard parameter-sweep techniques to obtain the learning rate of $0.1$, $120$ trees, a maximum tree depth of $5$ and a maximum number of leaves of $120$. 
This model is trained using $10$-fold cross-validation on $80\%$ of the the training data and an additional $10\%$ for validation; it achieves an AUC of $0.66$ and logistic loss of $0.45$, averaged over all training folds. 
Training this model is particularly challenging because of the small number of ``positive'' examples (i.e., the number of donor MAs).
Figure~\ref{fig:predictionScores} shows the density of prediction scores returned from this model, over all training data.
Most prediction scores are between $0$-$10\%$, with an average of $3.43\%$---quite close to the observed likelihood of MA. 

We use this model to estimate \emph{how likely} it is that a donor will take action, when notified about a particular donation opportunity. 
Next we describe how this model is used to design a notification policy: by framing blood donor recruitment as a \emph{matching} problem.

\section{Matching Framework for Blood Donation}\label{sec:framework}
We represent a blood donation problem as a weighted bipartite \emph{donation graph} $G = (U,V,E)$, with donors $u \in U$ and donation opportunities (or \emph{recipients}) $v \in V$.\footnote{We use the terms ``donors" and ``recipients" as shorthand for \emph{prospective} donors and recipients. 
Facebook does not make any determination about a person’s eligibility to donate blood; these are potential donors who sign up to receive notifications of blood donation opportunities.} 
Each vertex has a set of \emph{attributes} (e.g., blood type, geographical location, and so on), and these attributes determine whether a donor $u$ can donate to a recipient $v$---i.e., whether $u$ and $v$ are \emph{compatible}. 
Compatible pairs $(u,v)$ are connected by edges $e=(u,v)\in E$; we denote all edges adjacent to vertices $u\in U$ ($v\in V$) as $E_{u:}$ ($E_{:v}$).

If an edge $e=(u,v)$ exists, then donor $u$ can be \emph{notified} about $v$.\footnote{In this initial work, we assume the set of potential donors and donation centers do not change, although this \emph{longer-term dynamism} is certainly interesting to consider as future research.} 
We discretize time into days $t\in \mT \equiv \{1,\dots T\}$, with a finite-time horizon $T$.
In our setting both donors an recipients are dynamic, in the sense that some donors and recipients are available at certain time steps.
This notion of dynamism is designed specifically to represent a blood donation setting.

We assume that donors may receive only one notification at each time step, however \emph{any number} of donors may be notified about the same recipient on any time step.
Thus, our setting more-closely resembles $b$-matching~\cite{anstee1987polynomial} than traditional bipartite matching.

\paragraph{Edge Weights:} 
Each edge $(u,v)$ has weight equal to the \emph{probability} that donor $u$ donates to recipient $v$ once notified (i.e., the predicted MA likelihood, see \S\ref{sec:edgeweights}); we assume that edge weights $w_{et}$ are indexed by edge $e$ and time step $t$.
In other words, some edges (notifications) are more likely than others to result in donation: for example, certain people may be more likely than others to donate (e.g., people who have donated frequently in the past, as observed by~\cite{godin2007determinants}) and people may prefer to donate on specific days more than others.

\paragraph{Recipients:} We consider both \emph{static} recipients $S\subseteq V$, such as blood banks and hospitals, and \emph{dynamic} recipients (or \emph{events}) $D\subseteq V$, such as blood drives or emergency requests. 
Static recipients are available during \emph{all} time steps, and edges into these recipients are always available.
\emph{Events} arrive in an online manner, and are available only during certain time steps.
We assume that the \emph{distribution} of recipient availability is known and defined by $p_{vt}\in [0, 1]$: the probability that recipient $v$ is available at time $t$. 
The distribution of recipient arrivals $p_{vt}$ is assumed to be known; this is a primary input to our matching algorithms.
We use $\hat{p}_{vt}$ to denote a \emph{realization} of recipient arrivals, which is $1$ if donor $v$ is available at time $t$ and $0$ otherwise.
We assume that realized recipient arrivals $\hat p_{vt}$ are revealed on each time step $t$.
In other words, at time step $t'$ all realized arrivals $\hat p_{vt}$ are known for time steps $t$ with $1\leq t \leq t'$. 

\paragraph{Donors:} 
After a donor signs up with the Facebook Blood Donation Tool, we say they are \emph{available} to receive notifications (i.e., to be matched) at any time.
While there is essentially no limit on the number of notifications that can be sent on via online platform, there is a legal limit on how often people can donate blood. 
This limit is meant to protect donor health, and is often set by local governments or health authorities.\footnote{Typically 8 weeks or longer; see \url{https://www.redcrossblood.org/faq.html}.}
Thus, due to legal and health considerations, and out of respect for donors' time and attention, we limit how often each donor is notified: this limit is one notification every $K\in \mathbb Z_+$ days.
Since not all notifications lead to donation, it is reasonable to set $K$ to 7 or 14 days---much shorter than the donation rate limit.

\paragraph{Balancing Priorities:}
In general there are several priorities when matching blood donors and recipients: we aim to increase the number of active blood donors, maximize the number of donations, respect donor privacy and preferences, satisfy recipients' needs, and so on. 
Deciding which of these policies is most important is a matter of policy, and is beyond the scope of this paper. 
Here we consider two priorities which we believe are relevant to any blood donor matching platform: (a) increasing the overall number of donations from a fixed donor pool, and (b) treating recipients equitably.
While the justification for priority (a) is perhaps obvious, priority (b) requires more discussion.

\subsection{Equitable Treatment of Recipients}\label{sec:fairness}
In an online blood donor matching platform, notification policies have a far greater potential to impact recipients than donors.
From a donor's perspective, a change in notification policy might mean that they receive notifications at a slightly different rate, or that they are encouraged to donate to a different recipient.
(Recall that donors can always browse for opportunities using the Blood Donation tool; they need not pay attention to notifications.)
However from a recipient's perspective, a change in notification policy can drastically impact the number of notifications encouraging donors to visit their facility.
For example if predictive models suggest that edge weights to centrally-located hospitals are high, while edge weights to rural hospitals are near zero, then a simple edge-weight-maximizing policy would never notify donors about rural hospitals (indeed we report a similar distance-based effect in Section~\ref{sec:results}).
Furthermore, two-sided matching platforms---such as the Facebook Blood Donation tool---are most effective when both sides of the market benefit from participating.
If donors are never notified about rural recipients then these recipients might choose to leave the platform, which is a strictly worse outcome for everyone.
For these reasons we consider the \emph{fairness} of different notification policies. 

Our approach is inspired by the problem of \emph{fair division} in economics~\cite{steihaus1948problem}, and specifically the notion of weighted proportional fair division~\cite{crowcroft1998differentiated}.
In weighted proportional fair division, a finite set of resources is divided among agents such that each agent values their allocation proportional to their \emph{weight}---where greater weight represents greater endowment or priority.
In our setting, different recipients have different numbers of compatible donors (e.g., due to their location), or different edge weights (e.g., due to donor preferences or recipient accessibility); it may not be reasonable to, for example, guarantee that each recipient is matched with the same total edge weight. 
Instead we endeavor to match each recipient with edge weight proportional to their \emph{normalization score}---where normalization scores are provided as input to the matching policy.
Furthermore, since individual edges cannot be divided between recipients, it is not always possible to guarantee exact \fairness{} for all recipients. 
Instead we use a relaxed notion of proportionality, based on the normalized edge weight matched with each recipient.

\begin{definition}[$\gamma$-\Fair{} Matching]
Let $Y_v$ be the total weight matched with recipient $v$ over time horizon $\mT$, and let $m_v$ be the normalization score for $v$.
This matching is $\gamma$-\fair{} for $\gamma \in (0, 1]$ if 
$$
\gamma \frac{Y_{v'}}{m_{v'}} \leq \frac{Y_{v}}{m_v} 
$$
for each $v, v'\in V$. 
\end{definition}
%
%
In other words, a matching is $\gamma$-\fair{} if the normalized matched weight for recipient $v$ is at least fraction $\gamma\in (0, 1]$ of the normalized matched weight for all other recipients.
%
%
Note that with $\gamma=1$, all recipients receive the same normalized matched weight.

By this definition, it is always $\gamma$-\fair{} to allocate \emph{zero} matched weight to all recipients (i.e., $Y_v=0$ for all $v\in V$); we refer to this the \emph{empty} allocation.
We are interested in non-empty allocations; thus, one might wonder how hard it is to find \emph{any} $\gamma$-\fair{} allocation which matches at least one edge.
We refer to this as the $\gamma$-\fair{} allocation problem.
\begin{definition}[$\gamma$-\Fair{} Allocation Problem]
Input: $\gamma\in (0, 1]$, donation graph $G=(U,V,E)$, edge weights $w_e\in \mathbb R_+$ for each $e \in E$, and normalization scores $m_v\in \mathbb R_+$ for each $v\in V$.
All recipient availability is known ahead of time.
\emph{Does there exist a non-empty set of edges in $E'$, with $E'\subseteq E$, which covers each donor at most once, and is $\gamma$-\fair{} to all recipients?}
\end{definition}
\begin{theorem}\label{thm:hardness}
The $\gamma$-\fair{} allocation problem is NP-hard for every $\gamma \in (0, 1]$.
\end{theorem}

In other words, it is intractable to identify a $\gamma$-proportional allocation when recipient availability is known.
Furthermore, recipient availability is often unknown: some recipients may host regular week-long blood drives, and others may only accept donation in response to patient needs. 
%
%
%
Instead we focus on \fairness{} in \emph{expectation}---over all possible realizations of recipient availability.

\section{Matching Policies}\label{sec:matching-policies}

We aim to match donors with recipients such that we maximize edge weight (maximize the number of MAs), such that the outcome is $\gamma$-\fair{} for recipients. 
Here we define matching policies which trade off between both of these goals.
These policies assume that donor availability is \emph{fixed}, that is, we are given as input the time steps in which each donor can be notified.
This is a natural constraint for fielded notification systems, which may only notify donors, for example, on certain days of the week.
In the Electronic Companion (\ref{app:rate-limit}) we briefly discuss policies which also select \emph{when} to notify each donor.

Each matching policy takes as input a bipartite graph $G=(U, V, E)$ with edge weights $w_{et}$, normalization scores $m_v$, recipient arrival distribution $p_{vt}$, and time horizon $\mT$.
At each time step $t$, all observed demand realizations $\hat p_{vt'}$ for all $t' \leq t$ are ``revealed'' to the policy, and may be used as input.

We use parameters $a_{ut}$ to denote the (exogenous) donor availability on each time step: donor $u$ may be matched on time step $t$ only if $a_{ut}=1$.
We denote the set of available edges for recipient $u$ on time $t$ by $E_{u:}^t\equiv \{(u', v')\in E \mid u'=u, a_{ut}=\hat p_{v't}=1\}$.

In order to benchmark practical matching policies, we compare them with an unrealistic \emph{offline optimal} policy, which has complete knowledge of the ``true'' demand realization $\hat p_{vt}$. 
The offline optimal policy is defined using any optimal solution to Problem~\ref{eq:fixedtime-offline-opt}.
\begin{equation}
\label{eq:fixedtime-offline-opt}
    \begin{array}{rll}
        \max & \sum\limits_{t\in \mT}\sum\limits_{e\in E} w_{et} x_{et}\\
        \text{s.t.} & x_{et} \in \{0,1\} & \forall e\in E\; t\in \mT \\
        & s_{v} \in \mathbb{R} &\forall v\in V\\
        & x_{et} \leq \hat p_{vt} a_{ut} & \forall e=(u,v) \in E, \,t\in \mT\\
        & \sum\limits_{e\in E_{u:}^t} x_{et} \leq a_{ut} & \forall u\in U,\, t\in \mT \\
        & s_{v} = \frac{1}{m_{v}} \sum\limits_{t\in \mT}\sum\limits_{e\in E^t_{:v}} x_{et} w_{et} & \forall v\in V\\
        & \gamma s_{v} \leq s_{v'} & \forall v,v'\in V,\, v\neq v'.\\
    \end{array}
\end{equation}
Here variables $x_{et}$ are $1$ if edge is matched at time $t$ and $0$ otherwise; auxiliary variables $s_v$ denote the normalized matched weight for recipient $v$. 
An offline optimal policy for this setting is defined using an optimal solution to Problem~\ref{eq:fixedtime-offline-opt}.
\begin{definition}[Offline Optimal Policy \OPT{$\gamma$}]
Let $x_{et}^*$ be an optimal solution to Problem~\ref{eq:fixedtime-offline-opt}, for demand realization $\hat p_{et}$.
At each time $t\in \mT$, \OPT{$\gamma$} matches all edges $e \in E$ such that $x_{et}^*=1$.
Policy \OPT{$0$} refers to the offline-optimal matching policy without \fairness{} constraints.
\end{definition}
\begin{corollary}
It is NP-hard to identify policy \OPT{$\gamma$}, for every $\gamma\in (0, 1]$.
\end{corollary}
As a direct corollary of Theorem~\ref{thm:hardness}, Problem~\ref{eq:fixedtime-offline-opt} is NP-hard for every $\gamma\in (0, 1]$.
Thus, even if the demand realization is known, it is computationally hard to find an optimal matching.
Of course, in realistic settings the demand realization is not known.
Instead, our proposed policies use distributional information (exogenous parameters $p_{vt}$) to match donors and recipients. 
We compare these realistic policies to $\texttt{OPT}(\gamma)$ using two evaluation metrics:

\noindent\textbf{Competitive Ratio. } 
Let $E[\OPT{0}]$ be the expected matched weight by $\texttt{OPT}(0)$, over all demand realizations. 
Let $E[\texttt{ALG}]$ be the expected matched weight by matching policy \texttt{ALG}, over all demand realizations and (if \texttt{ALG} is stochastic) all policy realizations. The competitive ratio is 
$$ CR \equiv \min_{G, p, a} \frac{E[\texttt{ALG}]}{E[\OPT{0}]},$$
where the minimization is over all possible matching graphs, demand distributions, and donor availability.
In other words, $CR$ is the \emph{worst-case} ratio of expected matching weight over all possible matching scenarios.

\noindent\textbf{Expected Proportionality. }
Let $E[Y_v]$ be the expected weight matched by an a matching policy, over all demand realizations and (if \texttt{ALG} is stochastic) all policy realizations.
The expected \fairness{} of policy \texttt{ALG} is 
$$ EP \equiv \min_{G, p, a} \max\limits_{\gamma \in [0, 1]} \{\gamma E[Y_v] / m_v \leq E[Y_{v'}] / m_{v'} \,\, \forall (v, v') \in V, \,\, v \neq v' \}, $$
where as before $m_v$ is a fixed normalization score for recipient $v$, and the minimization is over all possible graphs, demand distributions, and donor availability.
In other words, if policy \texttt{ALG} is guaranteed to be $\gamma$-\fair{} in expectation then $EP=\gamma$.
Note that $EP$ may be $0$, meaning that there is no $\gamma>0$ such that the expected outcome is $\gamma$-\fair{}.
%

For the remainder of this section we assume that agent normalization scores are determined by a uniform random notification policy, defined below.
\begin{definition}[Uniform Random Policy \RAND{} (fixed-time)]
At each time step $t\in \mT$, for each available donor $u$: \RAND{} matches $u$ using an edge in $E^t_{u:}$ chosen uniformly at random.
\end{definition}
%
\begin{definition}[Normalization Score $m_v$]
Let $E[Y_v]$ be the expected weight matched with recipient $v$, over all recipient demand realizations and (for randomized policies) over all policy realizations. 
The scaling factor for recipient $v$ is $m_v \equiv E[U_v]$.
\end{definition}
Using these normalization scores we imply that policy \RAND{}, and its outcome, are ``fair''; we emphasize that this is only one choice of normalization scores, and in practice the notion of fairness/proportionality should be defined by stakeholders.

Metrics $CR$ and $EP$ help us characterize the expected performance of fixed-time matching algorithms.
In the following two sections we analyze two classes of policies: \emph{myopic} policies use only information from the current time step to make matching decisions (this includes both policies implemented in our online experiments); \emph{non-myopic} policies take into account the demand distribution for future time steps.

\paragraph{Myopic Policies} only take into account the information available at each time step. 
We consider two simple baseline myopic policies, \MAX{} and \RAND{} (defined above).
Policy \MAX{} is defined below.
%
%
\begin{definition}[Max-Weight Policy \MAX{}]
At each time step $t\in \mT$, for each available donor $u$: let $W\equiv \max_{e\in E^t_{u:}} w_{et}$ be the maximum edge weight for any of $u$'s available edges at time $t$. \MAX{} matches $u$ using any edge in $E^t_{u:}$ with edge weight $W$, and if multiple edges have weight $W$ then one is chosen randomly.
\end{definition}

First, note that \RAND{} has $EP=1$ by definition.
%
On the other hand, \MAX{} does not.
%
\begin{lemma}\label{lem:fixedtime-max-0-fair}
\MAX{} is $EP=0$; that is, in the worst case \MAX{} is $0$-\fair{} in expectation.
\end{lemma}
Intuitively \MAX{} ignores normalization weights $m_v$, meaning that it does not guarantee proportionality. 
In the worst case, \MAX{} can leave some recipients can unmatched, meaning that $EP=0$.
On the other hand, \MAX{} \emph{always} maximizes matched weight.
\begin{lemma}\label{lem:max-opt-fixedtime}
\MAX{} achieves competitive ratio $CR= 1$.
Further, without \fairness{} constraints ($\gamma=0$), \MAX{} is equivalent to an offline-optimal policy (\OPT{$0$}).
\end{lemma}
On the other hand, since \RAND{} ignores edge weight, its worst-case competitive ratio is low.
\begin{lemma}\label{lem:cr-rand}
\RAND{} achieves a competitive ration of at most $CR=1/N$ when there are $N$ recipients.
\end{lemma}

Baseline policies \MAX{} and \RAND{} represent two ends of a spectrum: on one side, \MAX{} prioritizes maximizing edge weight, at the cost of \fairness{} for recipients; on the other side, \RAND{} treats all recipients ``fairly'' (for one specific notion of fairness), but does not prioritize edge weights.
To balance these objectives in a principled way, we might randomly choose between \MAX{} and \RAND{} at each time step, for each donor. 
This is the purpose of myopic policy \texttt{RandMax}, defined below.
\begin{definition}[Hybrid Policy \texttt{RandMax}($gamma$)]
At each time step $t\in \mT$, and for each available donor $u\in U$, this policy randomly chooses to (a) match the donor using policy \MAX{} (with probability $1-\gamma$), or (b) match the donor using policy \RAND{} (with probability $\gamma$).
\end{definition}

Since this policy randomly mixes \MAX{} (which is equivalent to an offline-optimal policy with $\gamma=0$), and \RAND{} (which is a ``perfectly'' \fair{} policy in this setting), this hybrid policy effectively balances the objectives of maximizing matched weight and \fairness{} for recipients.
\begin{lemma}
\texttt{RandMax}($\gamma$) has $CR= 1-\gamma$ and $EP =\gamma$, for all $\gamma\in [0, 1]$.
\end{lemma}
In other words, this hybrid policy strikes a balance between matched weight ($CR$) and \fairness{} ($EP$), set by parameter $\gamma$.
%
%
However this hybrid policy may not be Pareto optimal: for $\gamma\in (0, 1)$ there may be another policy with stronger guarantees on both \fairness{} $EP$ and competitive ration $CR$.

We leave the task of identifying a Pareto optimal policy to future work; instead we propose a class of stochastic policies with moderate guarantees on $CR$ and $EP$, though their performance is far better than these guarantees in computational experiments.

The policies introduced in this section are based on the optimal solution to an LP formulation of our matching problem. 
As a baseline for these policies we use an LP relaxation of the offline optimal MILP, Problem~\ref{eq:fixedtime-offline-opt}.
We refer to this relaxation as Problem~\ref{eq:fixedtime-offline-opt}-LP (not stated explicitly). This problem is nearly identical to Problem~\ref{eq:fixedtime-offline-opt}, with two differences: (1) variables $x_{et}$ are continuous (on interval $[0, 1]$) rather than binary, and (2) demand realization $\hat p_{vt}$ is replaced by demand distribution $p_{vt}$.

Before defining matching policies based on Problem~\ref{eq:fixedtime-offline-opt}-LP, we make some important observations.
First, Problem~\ref{eq:fixedtime-offline-opt}-LP yields a valid upper bound for Problem~\ref{eq:fixedtime-offline-opt}

\begin{lemma}\label{lem:lp-ub}
Let $Z_{LP}$ denote the optimal objective of Problem~\ref{eq:fixedtime-offline-opt}-LP for a matching problem defined by $U,V,E,m_v,p_{vt},\mT$ and $\gamma \in [0, 1]$. 
Let $E[\OPT{\gamma}]$ be the expected objective of the offline-optimal policy, over all demand realizations.
Then, $ Z_{LP} \geq E[\OPT{\gamma}]$.
\end{lemma}

This result lets us use Problem~\ref{eq:fixedtime-offline-opt}-LP as an upper-bound on the matched weight for any matching policy; we use this as a baseline for which to compare other matching policies.

We consider two classes of LP-based policies: \emph{non-adaptive} policies (which pre-commit to a set of edges that may be matched), and \emph{adaptive} policies (which may change their matching decisions at each time step).

\subsection{Non-adaptive Policies}
We consider a class of non-adaptive policies which \emph{pre-match} at most one edge for each donor at each time step---that is, matching decisions may not adapt at each time step as new information is revealed.
At each time step, if the donor is pre-matched to an edge and the edge's recipient is available, then this edge is matched; otherwise the donor remains unmatched during this time step.
Of course, this does not guarantee that all donors are matched at each time step---and therefore the competitive ratio can be quite low.

\paragraph{Warm-Up: Policies based on Problem~\ref{eq:fixedtime-offline-opt}.} First we consider a non-adaptive policy based on an optimal solution for Problem~\ref{eq:fixedtime-offline-opt}-LP. 
\begin{definition}[\texttt{NAdapLP}($\alpha, \gamma$)]\label{def:NAdapLP}
Let $x^*_{et}$ denote an optimal solution to Problem~\ref{eq:fixedtime-offline-opt}-LP with \fairness{} parameter $\gamma\in [0, 1]$ and $\alpha\geq 0$.
For each time step $t\in \mT$ and each donor $u\in U$, edge $e\in E_{u:}$ is pre-matched with probability $\alpha x^*_{et} / p_{vt}$, and the donor is not pre-matched with probability $1 - \alpha \sum_{e=(u,v)\in E_{u:}} x_{et}^*/p_{vt}$.
At each time step, all donors are matched using their pre-matched edge, if the pre-matched donor is available.
\end{definition}

In this policy, parameter $\alpha$ is a scaling factor used to ensure that each edge assignment distribution is \emph{valid}---that is, that $ \alpha \sum_{e=(u,v)\in E_{u:}} x_{et}^* /p_{vt}\leq 1$ for all $u\in U$.
Note that this policy can only be implemented if each of these distributions are valid.
Conveniently, the probability that any edge is matched by \texttt{NAdapLP}($\alpha,\gamma$) is expressible in terms of the optimal solution to Problem~\ref{eq:fixedtime-offline-opt}-LP used to define this policy.
\begin{lemma}\label{lem:prob-match-nonadapt-fixedtime}
 Let $x^*_{et}$ be the optimal solution used in policy \texttt{NAdapLP}($\alpha,\gamma$).
 The unconditional probability that edge $e$ is matched at time $t$ by policy \texttt{NAdapLP}($\alpha,\gamma$) is $\alpha x^*_{et}$.
\end{lemma}

Lemma~\ref{lem:prob-match-nonadapt-fixedtime} leads to some additional observations about this policy.
\begin{corollary}\label{cor:fixedtime-weight}
 \texttt{NAdapLP}($\alpha,\gamma$) has competitive ratio $CR=\alpha$. 
\end{corollary}
\begin{corollary}
 \texttt{NAdapLP}($\alpha,\gamma$) is always $\gamma$-\fair{} in expectation, that is, $EP=\gamma$.
\end{corollary}
Both corollaries follow directly from Lemma~\ref{lem:prob-match-nonadapt-fixedtime} and the constraints of Problem~\ref{eq:fixedtime-offline-opt}-LP.
These results suggest that we can arbitrarily increase the weight matched by \texttt{NAdapLP}($\alpha,\gamma$) by increasing $\alpha$; however these policies are not guaranteed to be \emph{valid}.
This policy can only be implemented if $\alpha$ is small enough that each edge assignment distribution is valid.
\begin{lemma}\label{lem:nadaplp}
Policy \texttt{NAdapLP}($1/D,\gamma$) is always valid and achieves a competitive ratio of $CR = 1/D$ and $EP=\gamma$ for all $\gamma\in [0, 1]$, where $D$ is the maximum degree of any donor: $D \equiv \max_{u\in U} |E_{u:}|$.
\end{lemma}

In other words, Policy \texttt{NAdapLP}($1/D,\gamma$) is always implementable; thus there always exists a non-adaptive policy which achieves expected \fairness{} $EP=\gamma$ and competitive ratio $CR=1/D$ for all $\gamma \in [0, 1]$.
This competitive ratio guarantee is quite weak, and we might ask whether a better non-adaptive policy exists.
Indeed it does, and we discuss this policy next.

\paragraph{\textbf{Optimal $\gamma$-Fair Non-Adaptive Policies}}
Here we aim to identify a policy which is $\gamma$-\fair{} in expectation ($EP=\gamma$), and also maximizes matched weight (and thus $CR$); we refer to this as an \emph{optimal} $\gamma$-\fair{} non-adaptive policy.
To identify this policy, we first observe that \emph{any} non-adaptive policy can be characterized by the probability that it pre-matches edge $e$ at time $t$: $y_{et}\in [0, 1]$; using these statistics, the unconditional probability that $e=(u,v)$ is matched at time $t$ is $y_{et} p_{vt}$.
Note that for any non-adaptive policy, the probability that donor $u$ is pre-matched at time $t$ is at most $1$ if $u$ is available and $0$ otherwise; thus, statistics $y_{et}$ must satisfy conditions $\sum_{e\in E_{u:}} y_{et} \leq a_{ut}$ for all $u\in U$, and $t\in \mT$.
%
If a non-adaptive policy is $\gamma$-\fair{}, then $y_{et}$ must satisfy conditions
$$
    \gamma s_v \leq s_{v'} \quad \forall v, v'\in V
$$
with 
$$
    s_v = \frac{1}{m_v} \sum_{t\in \mT} \sum_{e\in E^t_{:v}} y_{et} p_{vt} w_{et} \quad \forall v\in V.
$$

Aggregating these conditions, we observe that the statistics $y_{et}$ of any $\gamma$-\fair{} non-adaptive policy is a feasible solution to the following LP.
\begin{equation}
\label{eq:optimal-nonadap-fixedtime}
    \begin{array}{rll}
        \max & \sum\limits_{t\in \mT}\sum\limits_{e\in E} w_{et} y_{et} p_{vt}\\
        \text{s.t.} & y_{et} \in [0, 1] & \forall e\in E\; t\in \mT \\
        & s_{v} \in \mathbb{R} &\forall v\in V\\
        & \sum\limits_{e\in E_{u:}} y_{et} \leq a_{ut} & \forall u\in U,\, t\in \mT \\
        &s_v = \frac{1}{m_v} \sum\limits_{t\in \mT} \sum\limits_{e\in E^t_{:v}} y_{et} p_{vt} w_{et} & \forall v\in V\\
        & \gamma s_{v} \leq s_{v'} & \forall v,v'\in V,\, v\neq v'.\\
    \end{array}
\end{equation}
Furthermore, a solution to Problem~\ref{eq:optimal-nonadap-fixedtime} corresponds to a non-adaptive policy; we use an optimal solution to this problem to define a $\gamma$-\fair{} non-adaptive policy.
\begin{definition}[\texttt{NAdapOpt}($\gamma$)]
Let $y^*_{et}$ be an optimal solution to Problem~\ref{eq:optimal-nonadap-fixedtime}.
For each time step $t\in \mT$ and each donor $u\in U$, a pre-matched edge is drawn with probability $y^*_{et}$; with probability $1 - \sum_{e\in E^t_{u:}} y^*_{et}$, no edge is pre-matched.
At each time step $t$ and for each available donor $u$, if the donor is pre-matched with an available recipient, then the pre-matched edge is matched.
\end{definition}

%
%
\begin{lemma}\label{lem:nadapopt}
%
%
\texttt{NAdapOpt}($\gamma$) 
%
%
achieves expected \fairness{} $EP=\gamma$ and maximal competitive ratio over all non-adaptive policies, with $CR\geq 1/D$ (where $D$ is the maximum degree of any donor).
\end{lemma}

Both non-adaptive policies described in this section are $\gamma$-\fair{} in expectation ($EP=\gamma$), thought their competitive ratio guarantee is somewhat weak. 
This is expected, since non-adaptive policies cannot change their matching decisions between time steps---they pre-match at most one edge for each donor at each time step.
Some pre-matched edges will in fact be unavailable, depending on the particular demand realization (which is not known in advance).

\subsection{Adaptive Policies}
%
Adaptive policies can use any available information in order to make matching decisions---including observed demand realizations, prior matching decisions, and the distribution of future demand.
We leave a general characterization of adaptive policies to future work; here we consider a simple class of adaptive policies that naturally extends the non-adaptive policies from the previous section.
This policy class, \texttt{AdaptMatch}, takes as input the set of edges pre-matched by a non-adaptive policy, denoted by $M$, where $M_{ut}=e\in E$ if $u$ is pre-matched along edge $e$ at time $t$, and $M_{ut}=\emptyset$ if $u$ remains unmatched at time $t$.
\texttt{AdaptMatch} uses pre-matched edges when possible, and if a pre-matched edge is not available it matches donors using either \RAND{} (with probability $\gamma$) or \MAX{} (with probability $1-\gamma$).
Algorithm~\ref{alg:adapt} gives a pseudocode description of this matching algorithm.


\SetKwInput{KwInput}{Input}                
\SetKwInput{KwOutput}{Output}              
\begin{algorithm}[H]
\SetAlgoNoLine
\KwInput{donors $V$, recipients $U$, edges $E$, time steps $\mT$, donor availability, pre-matched edges $M_{ut}$, parameter $\gamma\in [0, 1]$.}
\KwOutput{Matched edges at each time step}
 \For{each time step $t\in \mT$}{
    \For{each available donor, $u$}{
    \uIf{$u$ has a pre-matched edge $M_{ut}$, and this edge is available}{
        Match $u$ using pre-matched edge $M_{ut}$\;
    }
    \Else{
        Flip a weighted coin with ``heads'' probability $\gamma$\;
        \uIf{heads}{
        Match $u$ with policy \RAND{}\;
        }
        \Else{
        Match $u$ with policy \MAX{}\;
        }
    }

    }
 }
 \caption{\texttt{AdaptMatch}: Adaptive matching policy\label{alg:adapt}}
\end{algorithm}

Note that this adaptive policy matches strictly more edges (in expectation) than their non-adaptive counterparts.
Thus, expected matched weight (and $CR$) is strictly larger for \texttt{AdaptMatch} than the non-adaptive policy it is based on.

While competitive ratio is at least as large for these policies ($CR\geq 1/D$) as for their non-adaptive counterparts, there is no meaningful guarantee on expected \fairness{}.
We leave more sophisticated adaptive policies to future work.
However, while these approximate adaptive policies do not have strong guarantees on $CR$ or $EP$, they perform far better than these guarantees well in computational experiments (see \S~\ref{sec:simulations}).
%


\section{Results}\label{sec:results}

Prior to deploying new matching policies in an online setting, it is important to assess their performance in simulations.
Section~\ref{sec:simulations} outlines computational simulations with real data from the Facebook Blood Donation Tool, using our proposed matching policies; Section~\ref{sec:online-experiments} describes our online experiment with the Facebook blood donation tool.
In the Electronic Companion (\ref{app:simulations}) we also present results using synthetic, publicly available data.

\subsection{Computational Simulations}\label{sec:simulations}

We developed open-source simulation code for these simulations, which implements each of our proposed policies; details of these simulations are discussed in the Electronic Companion~\ref{app:simulations}.
All code used in these simulations is available in the supplementary material, and on Github.\footnote{Link removed during review. All code is included in the supplementary file \texttt{code.zip}.}
Data related to the Facebook blood donation tool cannot be released due to concerns for user privacy.
We test each matching policy from the previous section using data from the Facebook Blood Donation tool, and we ran separate simulations for 12 major cities around the world.
For each city we create a blood donation graph, consisting of donors $V$ and recipients $U$ registered with the Blood Donation tool; edges are created between donors and recipients within 15km of each other, and edge weights are calculated by the GBTD models described in Section~\ref{sec:edgeweights}.
Each of these cities has on the order of 1000 donors, 100 recipients, and 100,000 edges.

We require that donors are notified exactly once every $K=14$ days, and the first day each donor is notified is chosen randomly from $t\in \{1, \dots, 13\}$; recipient availability parameter $p_{vt}$ are determined from past notifications.
The realized recipient availability used in these experiments is randomly drawn using parameters $p_{vt}$, and this realization is fixed for the remainder of the experiment.
Each simulation runs for 60 days, so each donor is notified exactly 4 times.
Since policies \RAND{} and \texttt{AdaptMatch} are random, we run $50$ independent trials with these policies.
We define recipient normalization scores $m_v$ as the average weight matched to $v$ over all $50$ trials of \RAND{}.

For policy \MAX{} we calculate the total matched weight, and for \RAND{} and \texttt{AdaptMatch} we calculate the average matched weight over all trials.
We also calculate the (average) weight matched to each recipient, $Y_v$. 
Using the recipient weights we calculate a measure of \fairness{} $Gamma$, defined as
$$ Gamma \equiv \max \{\gamma \in [0, 1] \, | \, \gamma Y_v / m_v \leq Y_{v'}/m_{v'} \forall v, v' \in V\}.$$

\noindent\textbf{Simulation Results}
Simulation results for all $12$ cities are shown in Figure~\ref{fig:simulations}.
For each city we simulate matching using policies \MAX{}, \RAND{}, and \texttt{AdaptMatch}.
We implement several versions of \texttt{AdaptMatch}: each uses a fixed parameter $\gamma\in \{0.0, 0.1, \dots, 1.0\}$, and pre-matched edges $M\equiv \texttt{NAdapOpt}(\gamma)$.
These plots in Figure~\ref{fig:simulations} illustrate the trade-off between overall matched weight and proportionality (or fairness) for recipients.
\begin{figure}
    \centering
    \includegraphics[width=0.9\textwidth]{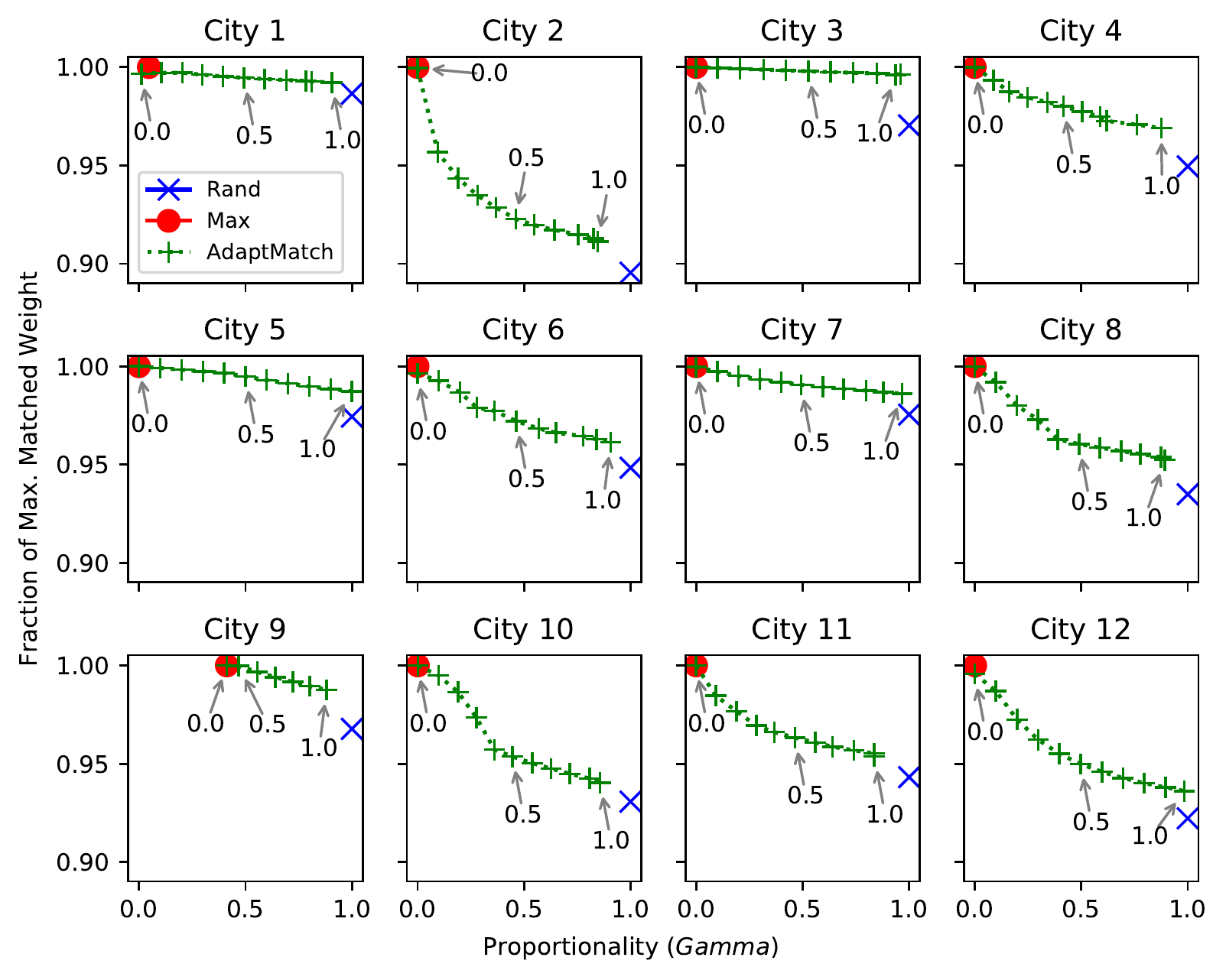}
    \caption{\label{fig:simulations}
    Simulation results for $12$ cities around the world.
    Each plot corresponds to one $60$-day trial in each city. 
    The vertical axis shows the fraction of matched weight, compared to \MAX{}; the horizontal axis shows \fairness{} metric $Gamma$.
    Policy \MAX{} is shown as a red circle, \RAND{} is a blue ``$\times$'', and \texttt{AdaptMatch} is a green ``+'' (for $\gamma=0.0, 0.1, \dots, 1.0$.
    Arrows on each plot indicate the values of $\gamma$ used by \texttt{AdaptMatch}.
    \label{fig:pareto-main}}
\end{figure}
While \MAX{} maximizes matched weight in this setting, it does not guarantee a \fair{} outcome: in all cities except for City 1 and City 9, $Gamma$ is zero for \MAX{}, meaning that some recipients are never matched by this policy. 
On the other hand, \RAND{} is \fair{} by definition (and $Gamma=1$), though this policy does not maximize matched weight.
However, \RAND{} always matches at least $90\%$ of the maximum possible matched weight in all simulations, and more than $95\%$ in five out of the 12 cities.

While policy \texttt{AdaptMatch} does not have strong guarantees on matched weight or \fairness{}, it mediates smoothly between  the extremes of \RAND{} and \MAX{}, according to parameter $\gamma$.
In some cases, this policy matches more weight than \RAND{}, while still achieving a nearly-\fair{} outcome ($Gamma$ equal to 1), as in Cities 3, 5, and 7.

\subsection{Online Experiments}\label{sec:online-experiments}

As a proof-of-concept, we compare the max-weight matching policy (\MAX{}) to the random baseline policy (\RAND{}, which is similar in behavior to the notification policy currently used by the Facebook Blood Donation tool), in an online experiment. 
The goal of this experiment is to answer the question: \emph{can we increase the overall number of donor meaningful actions} by carefully selecting \emph{which recipient} to notify each donor about. 
Both of these policies notify donors once every $14$ days; they only differ in \emph{which recipient} each donor is notified about.
\RAND{} selects a nearby recipient at random, while \MAX{} selects a nearby recipient with the greatest likelihood of donor MA---according to our predictive model.  

To compare these policies we design a randomized an online experiment, including hundreds of thousands of donors registered with the Facebook Blood Donation tool. 
We randomly partition these donors into a control group (who were notified using policy \RAND{}) and a test group (who were notified using policy \MAX{}). 
As in our simulations, we include \emph{only} static recipients (e.g., hospitals and large blood banks), who are always available to receive donations.

\paragraph{Potential Impact on Donors and Recipients.}
This experiment was approved by an internal review board.
We emphasize that the impact of these experiments is minimal: the only difference between the test and control group in this experiment is \emph{which} donation opportunity the donor is notified about.
The impact on blood recipients is less clear: due to our experimental design we cannot effectively measure the \fairness{} of each notification policy in a meaningful way.
However it is possible that any optimization-based matching policy (e.g., \MAX{} or \texttt{AdaptMatch}) prioritizes certain recipients over others.
This may marginalize recipients in rural areas or those with a limited Facebook presence.
More thorough analysis of these impacts is necessary before more widespread adoption of these policies.

\paragraph{Online Experiment Results}

This experiment ran from Nov.\ 23 to Dec.\ 17, 2019 (25 days); in total, 1,359,980 donors were notified using either policy \RAND{} or \MAX{}. 
In this experiment many donors had only one compatible recipient---in this case, the donor was \emph{always} notified about this recipient, regardless of the notification policy. 
For clarity, we distinguish between notifications sent to donors who had only one compatible recipient (1R), and those sent to donors with two or more compatible recipients (+2R).
Thus we only expect to observe a difference between control and test groups for +2R notifications; we expect the same outcome for (1R) notifications.
Table~\ref{tab:online-results} shows the number of notifications and meaningful actions for notifications of each type (1R and +2R), in both the test and control group.
Note that only +2R notifications are relevant for comparing the test and control groups, though we report both for transparency.
The key result in these tables is the percentage of notifications that led to meaningful action (\%MA, a number on $[0, 100]$). 
We report the Wilson score interval for \%MA as $C\pm R/2$, where $[C-R/2, C+R/2]$ is the 95\% confidence interval.

\begin{table}[]
\caption{\label{tab:online-counts} Online Experiments - Number of notifications (\#Notifs) and meaningful actions (\#MA), over the online experiment. Notifications are separated into those sent to donors with only one compatible recipient (1R), and those sent to donors with two or more compatible recipients (+2R). Wilson score intervals are for the percentage of notifications that lead to MA are presented as $C\pm R/2$, where the $95\%$ confidence interval is $[C-R/2, C+R/2]$.}
\renewcommand{\arraystretch}{1.0}
\begin{center}
\begin{tabular}{@{}crrrcrrr@{}}
\toprule
\multirow{2}{*}{Notif. Group} & \multicolumn{3}{c}{Control (\RAND{})} && \multicolumn{3}{c}{Test (\MAX{})} \\ 
  & \#MA    & \#Notifs  & \%MA && \#MA   & \#Notifs    & \%MA  \\ \midrule
 1R &10,534 & 215,544 & $4.7\pm 0.1$ && 10,755 & 214,841 & $4.8\pm 0.1$  \\
+2R & 15,551 & 420,230 & $3.7\pm 0.1$ && 16,054 & 412,387 & $3.9\pm 0.1$  \\\bottomrule 
\end{tabular}
\end{center}
\label{tab:online-results}
\end{table}

In the remaining discussion we consider only the +2R notifications, as there is no difference between the test and control group for 1R notifications.
For the overall experiment, \%MA is about $5\%$ higher for \MAX{} than for \RAND{}.
%
%
To better understand the differences between the control and test groups, we use two statistical tests to compare the notifications sent by \MAX{} and \RAND{}.

\paragraph{Overall Comparison} 
We use both a two-sided and one-sided Chi-square test to compare \%MA (+2R notifications only) for the control and test groups, over all notifications sent during this experiment.
Let $P_{\RAND{}}$ and $P_{\MAX{}}$ represent \%MA for the control (\RAND{}) and test (\MAX{}) groups, respectively.
The two-sided test checks the null hypothesis \textbf{H0:} $P_{\RAND{}}=P_{\MAX{}}$, with alternative $P_{\RAND{}}\neq P_{\MAX{}}$; the one-sided test checks null hypothesis \textbf{H0:} $P_{\RAND{}}=P_{\MAX{}}$, with alternative $P_{\RAND{}}<P_{\MAX{}}$. 
We can reject \emph{both} of these null hypotheses with $p\ll 0.01$. 
In light of the results presented in Table~\ref{tab:online-results}, these statistical test suggests \MAX{} achieves a small ($\sim 5\%$) but significant improvement over \RAND{} in terms of overall \%MA.
In the next set of statistical tests we compare each \emph{day} of the experiment as a separate trial.

\paragraph{Daily Paired Comparison} 
Next we treat day of the experiment as a set of \emph{paired measurements} of both $P_{\RAND{}}$ and $P_{\MAX{}}$.
For each day of the experiment ($26$ days in total) we calculate sample  estimates of $P_{\RAND{}}$ and $P_{\MAX{}}$---i.e., the $100$ times the ratio of MAs to overall notifications.
Note that donors are notified once every $14$ days, meaning that the set of donors notified on any particular day is nearly disjoint from the donors notified on any other day of the experiment; for this reason we treat the measurements of $P_{\RAND{}}$ and $P_{\MAX{}}$ on different days as independent.

We use a two-sided Wilcoxon signed-rank test to check the null hypothesis \textbf{H0:} the median difference between daily $P_{\MAX{}}$ and $P_{\RAND{}}$ is zero.
We reject this null hypothesis ($p\ll0.01$), further confirming that notification policy \MAX{} yields a higher MA rate than \RAND{}.
\begin{figure}
    \centering
  \includegraphics[width=0.7\linewidth]{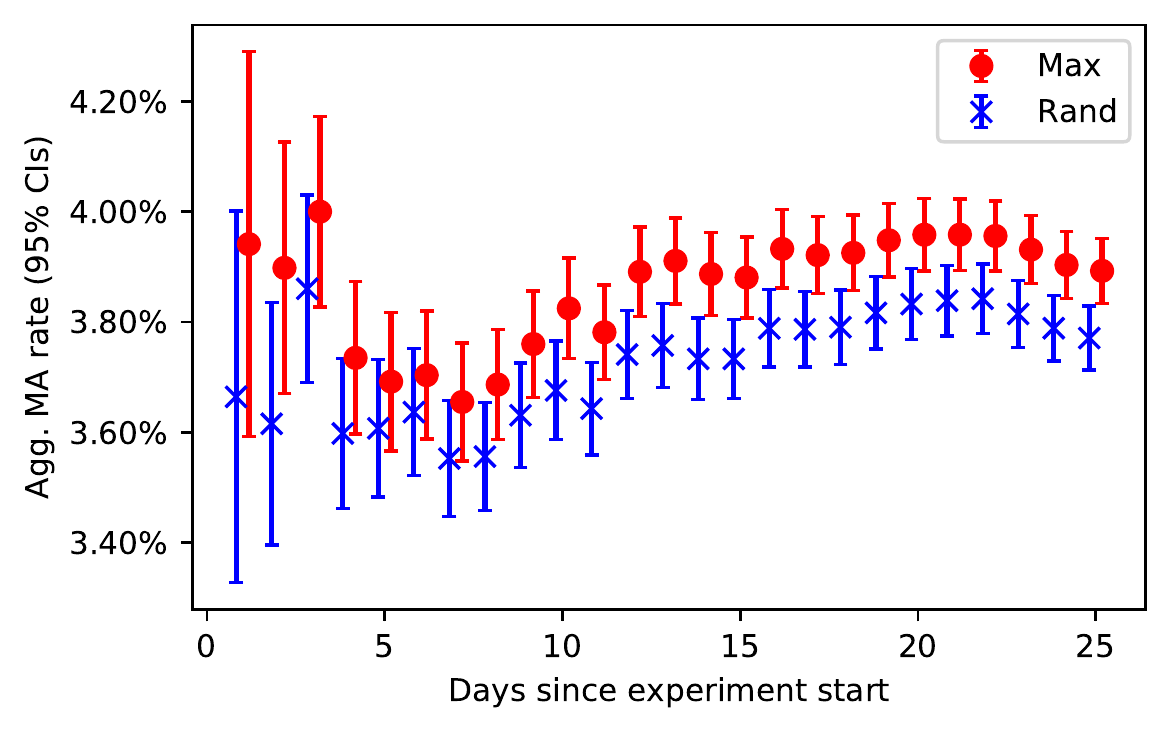}
\caption{
Aggregate MA rate for both \RAND{} and \MAX{}, for each day in the experiment. Rates are calculated using the cumulative number of notifications and MAs at each day in the experiment.
Error bars show the 95\% confidence interval (Wilson score interval), and points indicate the center of the interval.
%
}
  \label{fig:daily-ma} 
\end{figure}
For illustration, Figure~\ref{fig:daily-ma} shows the $95\%$ confidence intervals for $P_\RAND{}$ and $P_\MAX{}$, using the aggregated number of notifications and MAs for each day of the experiment.
In the Electronic Companion (\ref{sec:additional-experiments}) we show the results for each individual day, as well as the cumulative rates.

\section{Discussion}

We introduce the problem of connecting blood donors with demand centers in a time-dependent setting, with uncertain demand. 
We formalize this as an online matching problem, with the priorities of \emph{efficiency} (maximizing the number of donations) and \emph{fairness} (\fairness{}) for recipients. 
We propose a class of stochastic policies for this setting, to which we compare a realistic randomized baseline.
In simulations we see a clear trade-off between the overall number of donations and \fairness{} (Figure~\ref{fig:pareto-main}); the particular trade-off between these objectives depends on the notification policy used.
Policy \MAX{} (which maximizes edge weight/expected donations) results in a $5$-$10\%$ increase in the overall number of expected donations, compared to a random baseline (\RAND{}).
However \MAX{} tends to favor certain recipients over others.
In our simulations, \MAX{} completely ignores some recipients in $11$ out of the $12$ cities tested---presumably because these recipients are associated with lower edge weights.
On the other hand, \RAND{} always sends a ``fair'' amount of notifications to each recipient, regardless of edge weight (according to the definition of fairness and \fairness{} used in this study).
To mediate between the extremes of \RAND{} and \MAX{}, we propose a class of stochastic policies (\texttt{AdaptMatch}); in simulations these policies effectively control the balance between the overall expected number of donations and \fairness{} across recipients, using parameter $\gamma$.

As a proof-of-concept we run an online experiment via the Facebook Blood Donation Tool, comparing notification policies \RAND{} and \MAX{}.
We find that \MAX{} results in about $5\%$ more meaningful actions (a proxy for donations) than \RAND{}.
In relative terms this improvement seems small, however the implications are quite meaningful.
This experiment investigated \emph{one small improvement} to the notification strategy used by the Facebook Blood Donation Tool, i.e., whether the donor is notified about a nearby donation opportunity at random (\RAND{}), or notified about a particular opportunity selected by a predictive model (\MAX{}).
Several other modifications to the notification policy might yield similar improvements: for example by changing \emph{how often} each donor is notified, by more carefully planning for \emph{future donation needs}, or by tailoring notifications to each donor's unique preferences and values.

The potential impact of this work is considerable.
Indeed, if our observed results generalize to the entire community of Facebook blood donors, then a $5\%$ increase in donor action corresponds to at about $170,000$\footnote{Our results reported in Table~\ref{tab:online-results} suggest that policy \MAX{} leads has a meaningful action rate of $3.9\%$, compared to $3.7\%$ for policy \RAND{}. The difference is $0.2\%$---or $160,000$ of the estimated 85 million donors registered with the Blood Donation Tool (\texttt{\url{https://socialimpact.facebook.com/health/blood-donations/}}).} \emph{more} donors taking meaningful action toward donation when notified.
Even if few of these meaningful actions lead to actual donation, the increase is still substantial.

Before implementing these policies at a large scale in practice, it is important to understand their potential impacts on both blood donors and recipients.
In this study impact on donors is minimal; the only difference between notification policies is \emph{which donation opportunity} they are notified about.
However our simulation results indicate that blood recipients may face significant impacts from changes in notification policy.
For example policies that prioritize edges with a high likelihood of meaningful action (e.g., policy \MAX{}) may ignore certain recipients---such as  rural hospitals or small donation centers with a limited web presence.
%
%
%
This observation is particularly troubling if low-weight recipients are \emph{already} unlikely to recruit donors, which we expect is the case.
Of course, this potential injustice is exactly the motivation for our stochastic policy \texttt{AdaptMatch}.

Blood donation is a global challenge, and has been the focus of many dedicated organizations and researchers for decades. 
In this paper we investigate a new opportunity to recruit and coordinate a massive network of blood donors and recipients, enabled by the widespread use of social networks.
We formalize a matching problem around matching blood donors with recipients, and test these policies in both offline simulations and an online experiment using the Facebook Blood Donation Tool.
Our findings suggest that a matching paradigm can significantly increase the overall number of donations, though it remains a challenge to do so while treating recipients equitably.


\bibliographystyle{Science}

\bibliography{refs,new_refs} 

\textbf{Data and materials availability:} The code used for all computational simulations in this paper is available in the supplementary material, as well as on Github.\footnote{Link removed during review. All code is included in the supplementary file \texttt{code.zip}.}
Data related to the Facebook blood donation tool cannot be released due to concerns for user privacy.


\clearpage
\appendix



\section{Computational Simulations using Synthetic Data}\label{app:simulations}
Here we provide additional simulation results using publicly-available data.
All code used in this section is available online.\footnote{Link removed during review. All code is included in the supplementary file \texttt{code.zip}.}
We draw random donor and recipient locations from population distributions from four large cities around the world: Jakarta (Indonesia), Istanbul (Turkey), S\~ao Paulo (Brazil), and San Francisco (United States).
All population distributions are generated using data from the Socioeconomic Data and Applications Center (SEDAC) (\cite{sedac}); distance between each donor and recipient is calculated using the Haversine approximation.

\textbf{Edges:} Edges are created for all donor-recipient pairs within 15km of each other.
Edge weights are generated according to random attributes assigned to donors and recipients: each recipient is randomly assigned a ``nominal'' edge weight $w_0\sim U[0.01, 0.08]$, and each recipient is randomly assigned a decay parameter $k\in [5, 10, 20]$.
Edge weights are calculated using the expression $w_0 \times \exp(-D/k)$, where $w_0$ is the recipient's nominal edge weight, $k$ is the donor's decay rate, and $D$ is the distance between donor and recipient (in km).
These parameters are selected to roughly model the heterogeneity of real donation settings: some recipients are more popular or have a greater online presence than others (thus, higher $w_0$); some donors are more willing to travel long distances than others (thus, higher $k$). 

\textbf{Recipient availability:} Half of all recipients are randomly assigned to be static (always available), while the other half are dynamic. 
Dynamic recipients have availability parameters $p_{vt}$ generated as follows: we generate alternating sequences of \emph{low} probability ($p_{vt}=0.1$) and \emph{high} probability ($p_{vt}=0.9$); each sequence has random Poisson-distributed length, with mean $4$. 
These sequences are appended together to create $p_{vt}$ for all $t\in \mT$; the first sequence is randomly chosen to be low or high probability. 
For each matching scenario, we draw a single realization of recipient availability using parameter $p_{vt}$, and this realization remains fixed for the remainder of the experiment.

\textbf{Matching Simulation:} We simulate an online donation scenario over $30$ days, where each donor is notified exactly once every $7$ days; each donor receives their first notification on a random day between the first and sixth day, so each donor is notified either $4$ or $5$ times in each simulation.
We calculate recipient normalization scores by running $100$  trials of \RAND{}; normalization scores $m_v$ are the average weight matched with each recipient $v$ over all trials.

\textbf{Results:}
For each policy we calculate the total matched weight, and the fraction of the maximum possible weight, matched by policy \MAX{}. 
To report \fairness{} we first calculate the normalized weight for each recipient $Y_v/m_v$: the total weight matched with a recipient, divided by their normalization score).
For each policy we calculate a measure of \fairness{} $Gamma$, defined as:
$$Gamma\equiv \max \{\gamma \in [0, 1] \, | \, \gamma Y_v/m_v \leq Y_{v'} / m_{v'} \forall v, v'\in V\}.$$
That is, $Gamma$ is an empirical measure of \fairness{} for an allocation.

Figure~\ref{fig:simulations-app} shows simulation results for all four cities, with matching using policies \MAX{}, \RAND{}, and \texttt{AdaptMatch} (with $\gamma=0.0, 0.1, 0.2, \dots, 1.0$).

\begin{figure}
    \centering
    \includegraphics[width=\textwidth]{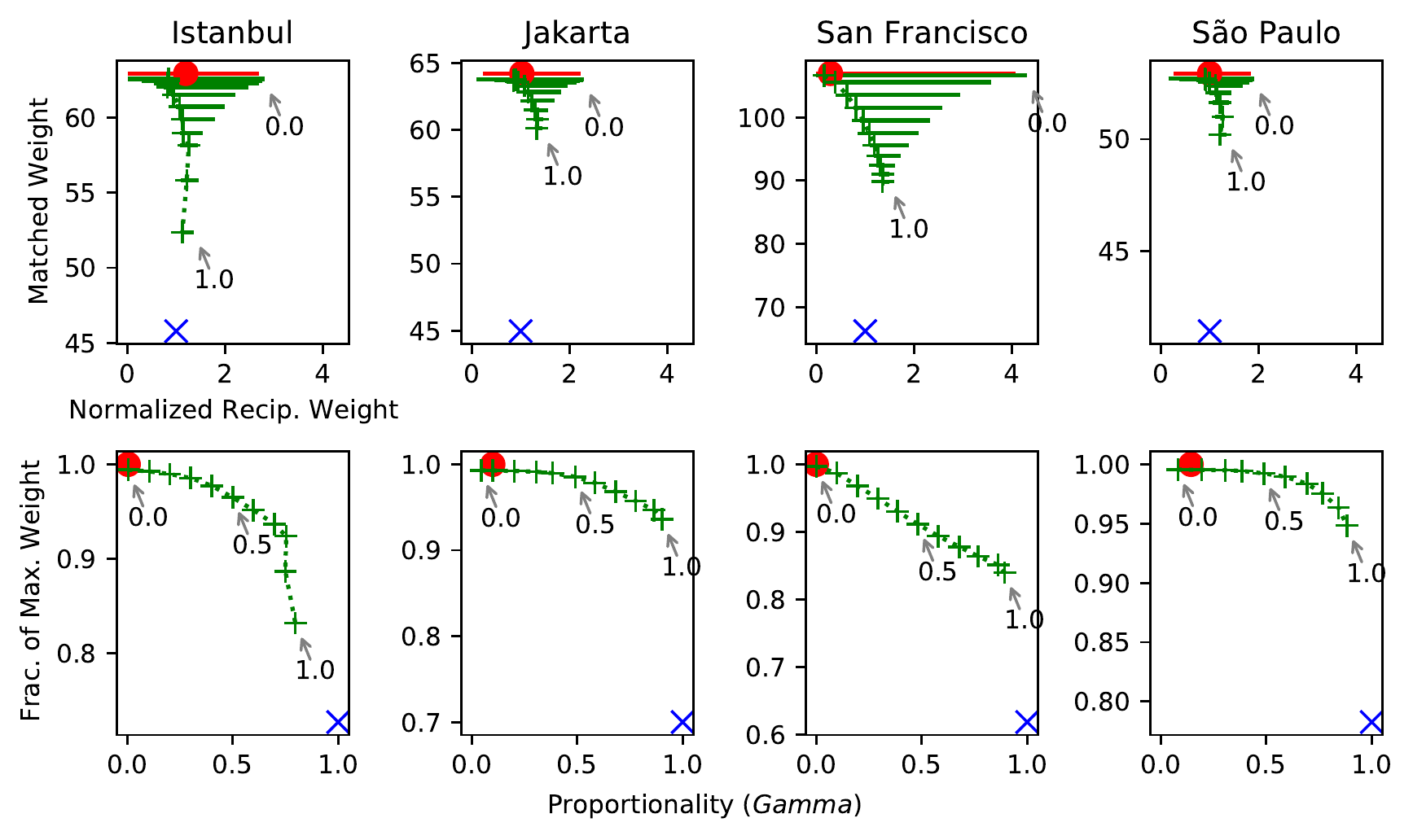}
    \caption{\label{fig:simulations-app}
    Simulation results for four cities, for matching policy \MAX{} (red circle), \RAND{} (blue ``$\times$'') and \texttt{AdaptMatch} with $\gamma=0.0, 0.1, \dots, 1.0$ (green ``+''). 
    Top Row: The vertical axis shows total matched weight for \MAX{}, and the average matched weight for \RAND{} and \texttt{AdaptMatch}; the horizontal axis shows the range of normalized recipient outcomes $Y_v/m_v$; the plot markers show the median value of the range.
    Bottom Row: The vertical axis shows total matched weight as a fraction of \MAX{}; the horizontal axis shows \fairness{} metric $Gamma$.
    Arrows on all plots indicate the $\gamma$ values for \texttt{AdaptMatch}.
    }
\end{figure}

The top row of this figure shows the total weight matched by each policy, and the normalized recipient outcomes; horizontal error bars show the range of normalized recipient outcomes.
A wider range corresponds to a less-\fair{} outcome, since some recipients receive much greater normalized matched weight than others.
For example in San Francisco, policy \MAX{} matches some recipients with normalized weight of $4$, while most other agents receive normalized weight near $0$.

The bottom row shows matched weight as a fraction of \MAX{}, and \fairness{} $Gamma$.
As expected, \MAX{} maximizes matched weight, though there is a wide range of recipient outcomes: for both Istanbul and San Francisco, at least one recipient remains unmatched by \MAX{} (and thus $Gamma$ is zero). 

On the other hand \RAND{} by definition guarantees a \fair{} outcome, with $Gamma=1$. 
This comes at a cost of matched weight: \RAND{} matches between $60\%$ and $80\%$ of the weight matched by \MAX{}.

Policy \texttt{AdaptMatch} mediates between these two extremes, varying the trade-off between weight and \fairness{} with parameter $\gamma$.

Our two primary observations from these experiments are (1) while policy \MAX{} maximizes matched weight, it clearly treats recipients unequally; in the worst case, some recipients are never matched; 
(2) while policy \RAND{} treats recipients equally, it results in a $20$-$30\%$ reduction in matched weight.
Policy \texttt{AdaptMatch} moderates smoothly between \MAX{} and \RAND{}, using parameter $\gamma$; often, this policy yields a Pareto improvement over both extremes.

\subsection{Real-World Online Experiments}\label{sec:additional-experiments}

Figure~\ref{fig:daily-ma-app} shows 95\% confidence intervals (Wilson score) for MA rate in the online experiment.
The top plot shows the aggregated MA rate, using the cumulative number of  notifications and MAs up to each day in the experiment.
The bottom plot shows MA rates for each individual day, using only notifications sent on each day.

\begin{figure}
    \centering
  \includegraphics[width=0.7\linewidth]{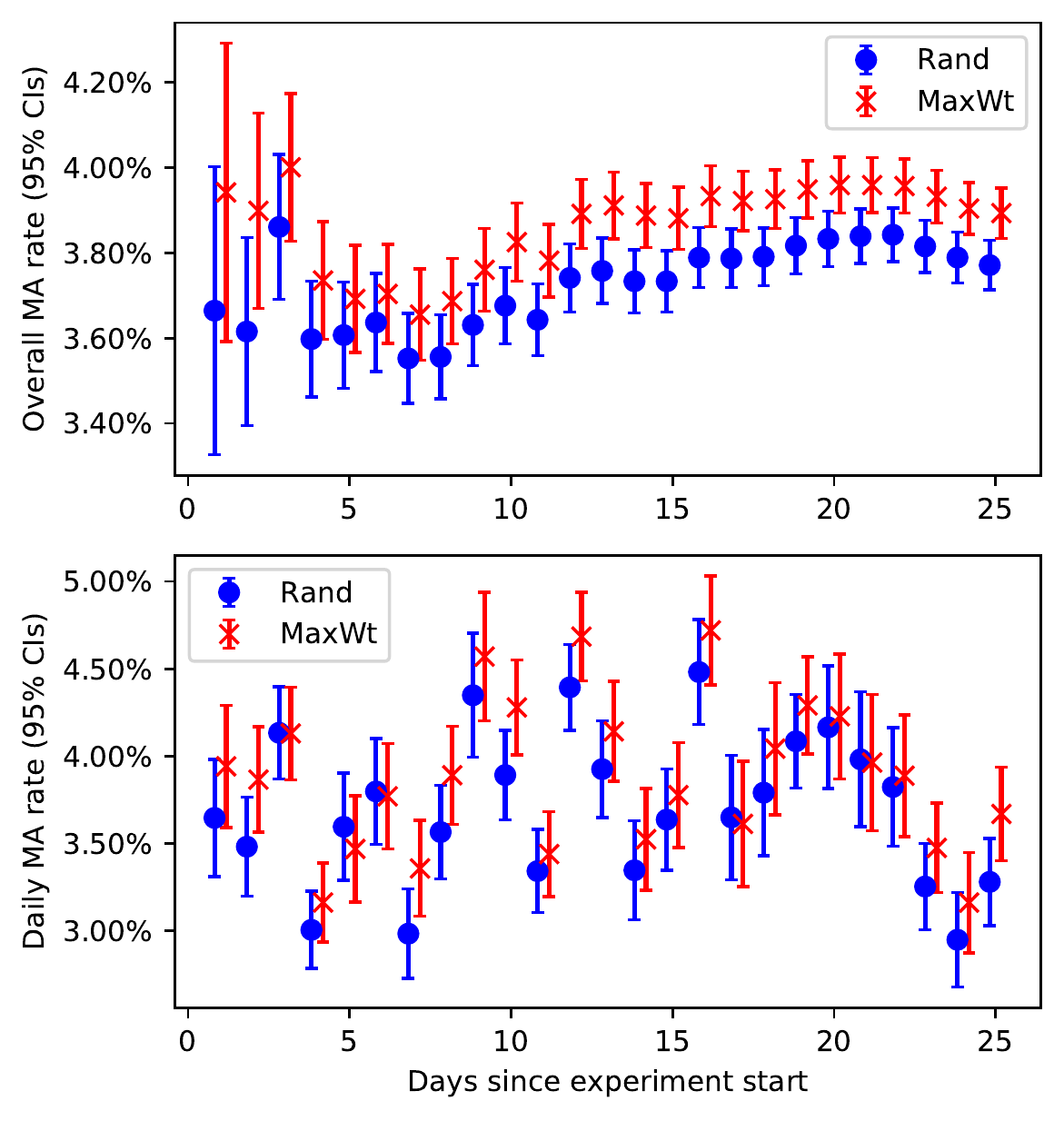}
\caption{
(Top) Aggregate MA rate for both \RAND{} and \MAX{}, for each day in the experiment. Rates are calculated using the cumulative number of notifications and MAs at each day in the experiment.
Error bars show the 95\% confidence interval (Wilson score interval), and points indicate the center of the interval.
(Bottom) Daily MA rates, calculated using only the MAs and notifications for each day.
}
  \label{fig:daily-ma-app} 
\end{figure}


\section{Proofs omitted from the main paper}\label{app:paper-proofs}

This section contains proofs of all theorems and lemmas omitted in the main paper.

\paragraph{Proof of Theorem~\ref{thm:hardness}}
\begin{proof}
This proof uses a reduction from $k$-EQUAL-SUM-SUBSET and PARTITION, each of which are defined as follows:

\paragraph{$k$-EQUAL-SUM-SUBSET:} 
given a multiset $\mathcal S$ of positive integers $x_1, \dots, x_N$, determine whether there are $k$ non-empty disjoint subsets $S_1, \dots, S_K\subset \mathcal S$ such that the sum of integers in each subset is equal. This problem is NP-complete for any $k>1$, and strongly NP-complete when $k$ varies as a function of $N$ and $k=\Omega(N)$~\cite{Cieliebak08:Complexity}.

\paragraph{PARTITION:} 
given a set $\mathcal S$ of positive integers $x_1, \dots, x_N$, determine whether there is a partition of $S$ into subsets $S_1,S_2\subset S$, with $S_1 \cup S_2=S$, such that the sum of $S_1$ and $S_2$ are equal. 
This problem is NP-complete, though efficient pseudo-polynomial time algorithms exist.

We consider two cases separately: $\gamma = 1$ and $\gamma \in (0, 1)$:

\begin{itemize}
    \item \textbf{Case 1:} $\gamma = 1$. \emph{reduction from $k$-EQUAL-SUM-SUBSET.} %
    Given an instance of $k$-EQUAL-SUM-SUBSET we construct a blood donor matching scenario as follows: let there be $k$ recipients (one for each subset) and $N$ donors (one for each integer $x_i$). 
    Each donor $i$ has edge weight $x_i$ to every recipient, thus $G$ is a complete bipartite graph. 
    Let all recipients have the same normalization score $m_v=1$.
    In this case a non-empty $\gamma$-\fair{} allocation awards the same matched weight to every recipient, since all recipients have the same normalization score.
    If such an allocation exists, it can be used to construct an equal-sum partitioning of integers $x_i, \dots, x_N$ into $k$ non-empty, disjoint subsets as follows: let $M_j$ be the set of donor indices matched with recipient $j$, and let subsets $S_1, \dots, S_k$ be defined as $S_j \equiv \{x_{i'} \mid i'\in \{1, \dots, N\}, i' \in M_j \}$; thus, $S_1, \dots S_k$ are non-empty disjoint equal-sum subsets of integers $S$.
    \item \textbf{Case 2:} $\gamma \in (0, 1)$. \emph{reduction from PARTITION.} 
    Given an instance of PARTITION we construct a blood donor matching scenario with $N+1$ donors and $3$ recipients.
    Donors $1$ through $N$ correspond to integers $x_1, \dots, x_N$, and recipients $1$ and $2$ correspond to subsets $S_1$ and $S_2$; as before, all recipient normalization scores are $m_v=1$.
    All donors $1$ through $N$ are adjacent to both recipients $1$ and $2$, where all edges adjacent to donor $i$ have edge weight $x_i$.
    Donor $N+1$ and recipient $3$ are adjacent \emph{only to each other}, with edge weight $\sum_i x_i/ (2 \gamma)$.
    In this case, a non-empty $\gamma$-\fair{} allocation \emph{must} match recipient $3$, resulting in normalized matched weight $\sum_i x_i / (2 \gamma)$.
    Due to \fairness{} constraints both recipients $1$ and $2$ must be matched with normalized matched weight at least $\sum_i x_i / 2$; thus, both recipients must be matched with \emph{exactly} edge weight $\sum_i x_i/2$.
    If such an allocation exists, it can be used to construct an equal sum partition: let $M_1$ and $M_2$ be the indices of donors matched with recipients $1$ and $2$, respectively; let subsets $S_1$ and $S_2$ be defined as $S_j \equiv \{x_{i'} \mid i'\in \{1, \dots, N\}, i\in M_j \}$.
    By definition, both $S_1$ and $S_2$ are equal-sum subsets of integers $S$, and $S_1 \cup S_2 = S$.
\end{itemize}
\end{proof}

\paragraph{Proof of Lemma~\ref{lem:fixedtime-max-0-fair}: $EP=0$ for \MAX{}}
\begin{proof}
We provide a simple example where \MAX{} is $0$-\fair{}.
Let there be one donor and two recipients ($A$ and $B$); the edge to recipient $A$ has weight $0.9$, while the edge to recipient $B$ has weight $1.0$.
Suppose there is only one time step.
\RAND{} matches recipient $A$ and $B$ with equal probability, while \MAX{} never matches $A$. 
Thus for policy \MAX{}, $E[Y_A]=0$ and $m_A>0$; this means that there is no $\gamma>0$ such that this outcome is $\gamma$-\fair{}.
\end{proof}




\paragraph{Proof of Lemma~\ref{lem:max-opt-fixedtime}: $CR=1$ for \MAX{}, and with $\gamma=0$, \MAX{} is equivalent to \OPT{$0$}}
\begin{proof}
First we show that the edges matched by \MAX{} are an optimal solution to Problem~\ref{eq:fixedtime-offline-opt} without \fairness{} constraints, meaning that \MAX{} is an optimal solution \OPT{$0$}.

\textit{Proof by contradiction.} Let $x_{et}$ be the decision variables representing edges matched by \MAX{} (i.e., $x_{et}$ is $1$ is $e$ is matched at time $t$ by \MAX{}, and $0$ otherwise).
Suppose that $x_{et}$ is not an optimal solution to Problem~\ref{eq:fixedtime-offline-opt}. 
Note that without \fairness{} constraints,  Problem~\ref{eq:fixedtime-offline-opt} can be decomposed by both donors $u\in U$ and time steps $t\in \mT$.
If $x_{et}$ is not an optimal solution, then there is a donor $u\in U$ and time $t\in \mT$ such that $\sum_{e\in E_{u:}^t} x_{et} w_{et}$ which is not optimal, i.e., $e$ is not a maximal-weight edge for donor $u$ at time $t$.
In this case, solution $x_{et}$ does not match a maximal-weight edge from $E^t_{u:}$, and thus $x_{et}$ was not produced by \MAX{}, a contradiction.
\end{proof}




\paragraph{Proof of Lemma~\ref{lem:cr-rand}: $CR=1/N$ for \RAND{}}
\begin{proof}
Consider an example donation graph with $N$ recipients and one donor; there is one edge from the donor to each recipient, and one time step during which all edges are available. 
One ``high-weight'' recipient has edge weight $1$, while the remaining $N-1$ ``low-weight'' recipients have edge weight $\epsilon \simeq 0$.
Policy \MAX{} matches the high-weight recipient with total weight $1$ (due to Lemma~\ref{lem:max-opt-fixedtime}, while \RAND{} matches all recipients with equal probability, with expected weight $1/N + \epsilon (N - 1)/N$.
As $\epsilon \rightarrow 0$, the expected matched weight of \RAND{} is $1/N$, and thus $CR=1/N$.
\end{proof}

\paragraph{Proof of Lemma~\ref{lem:lp-ub}: $Z_{LP}\geq E[\OPT{\gamma}]$}
\begin{proof}
Let 
$(x_{et}^*\mid \hat p_{vt})$ denote the optimal solution of Problem~\ref{eq:fixedtime-offline-opt} for demand realization $\hat p_{vt}$, and let
$\overline x_{et}^*$ denote the \emph{expected value} of $(x_{et}^*\mid \hat p_{vt})$ over all demand realizations drawn from distribution $p_{et}$. 
Note that $\overline x_{et}^*$ is a feasible solution to Problem~\ref{eq:fixedtime-offline-opt}-LP: by taking the expected value of both sides of all constraints in Problem~\ref{eq:fixedtime-offline-opt}, we exactly recover Problem~\ref{eq:fixedtime-offline-opt}-LP (note that, by definition, $E[\hat p_{vt}] = p_{vt}$).
Due to linearity of expectation, the expected objective of the offline optimal solution $(x_{et}^*\mid \hat p_{vt})$ is exactly equal to the objective of $\overline x_{et}^*$ in Problem~\ref{eq:fixedtime-offline-opt}-LP---we denote this expected objective by $E[\OPT{\gamma}]$.
In summary, the \emph{expected} solution to Problem~\ref{eq:fixedtime-offline-opt}, $\overline x_{et}^*$, is a feasible solution to Problem~\ref{eq:fixedtime-offline-opt}-LP and the expected objective value of Problem~\ref{eq:fixedtime-offline-opt} is exactly equal to the objective of $\overline x_{et}^*$ in Problem~\ref{eq:fixedtime-offline-opt}-LP. Therefore, $\texttt{LP}(\gamma) \geq E[\OPT{\gamma}]$.
\end{proof}

\paragraph{Proof of Lemma~\ref{lem:prob-match-nonadapt-fixedtime}: the unconditional probability of matching $e$ at $t$ with for \texttt{NAdapLP}($\alpha,\gamma$) is $\alpha x_{et}^*$}
\begin{proof}
Let $R^t_{v}$ be the event that recipient $v$ is available at time $t$, when using policy \texttt{NAdapLP}($\alpha,\gamma$).
Let $X^t_{uv}$ be the event that $u$ is matched by \texttt{NAdapLP}($\alpha,\gamma$) using edge $e=(u,v)$ at time $t$; note that $X^t_{uv}$ and $R^t_v$ are independent
By conditioning on $R^t_v$, the probability of $X^t_{uv}$ as follows 
\begin{align*}
    X^t_{uv} &= P[ X^t_{uv} | R^t_v] = \alpha\frac{x^*_{et}}{p_{vt}} p_{vt} \\
    &= \alpha x^*_{et}
\end{align*}

\end{proof}


\paragraph{Proof of Lemma~\ref{lem:nadaplp}: \texttt{NadapLP}($1/D,\gamma$) is always valid}
\begin{proof}
Corollary~\ref{cor:fixedtime-weight} states that the weight matched by \texttt{NAdapLP}($\alpha,\gamma$) is proportional to the optimal objective of Problem~\ref{eq:ratelimit-offline-opt}-LP, thus the competitive ratio of \texttt{NAdapLP}($\alpha,\gamma$) is $\alpha$. 
It remains to show that this policy is valid for $\alpha=1/D$.

Constraints in Problem~\ref{eq:ratelimit-offline-opt}-LP state that $x_{et}/p_{vt}\leq 1$; therefore $\sum_{e\in E^t_{u:}} x^*_{et}/p_{vt} \leq |E^t_{u:}| \leq D$ and $\frac{1}{D}\sum_{e\in E^t_{u:}} x^*_{et}/p_{vt} \leq 1$, meaning that this policy is valid for $\gamma = 1/D$.
\end{proof}

\paragraph{Proof of Lemma~\ref{lem:nadapopt}: $EP=\gamma$ and $CR\geq 1/D$ for \texttt{NAdapOpt}($\gamma$)}
\begin{proof}
First, since $y^*_{et}$ is a feasible solution for Problem~\ref{eq:optimal-nonadap-fixedtime}, Policy \texttt{NAdapOpt}($\gamma$) has expected \fairness{} $EP=\gamma$ due to constraints in Problem~\ref{eq:optimal-nonadap-fixedtime}.
Furthermore, if $y^*_{et}$ is an optimal solution, then the corresponding \texttt{NAdapOpt}($\gamma$) policy has both $EP=\gamma$, and maximal competitive ratio $CR$.

Since policy \texttt{NAdapLP}$(1/D,\gamma)$ achieves competitive ratio $CR=1/D$, it follows that \texttt{NAdapOpt\_Fixedtime} achieves a competitive ratio at least $1/D$. 
To further illustrate this, consider the pre-match distribution used by policy \texttt{NAdapLP}$(1/D)$: edge $e$ is matched at time $t$ with probability $\alpha x^*_{et}/p_{vt}$, where $x^*_{et}$ is an optimal solution to Problem~\ref{eq:fixedtime-offline-opt}-LP.
Note that $\overline y_{et}\equiv \frac{1}{D} \frac{x^*_{et}}{p_{vt}}$ is a feasible solution to Problem~\ref{eq:optimal-nonadap-fixedtime} (condition $\sum_{e\in E^t_{u}} \overline y_{et}\leq 1$ is met, due to constraints in Problem~\ref{eq:fixedtime-offline-opt}-LP).
Since this non-adaptive policy achieves $CR=1/D$, an optimal non-adaptive policy (corresponding to an optimal solution of Problem~\ref{eq:optimal-nonadap-fixedtime}) achieves competitive ratio $CR\geq 1/D$.
\end{proof}

\section{Rate-Limited Notification Policies}\label{app:rate-limit}

Rather than fixing the time steps when donors can be notified (``fixed time'' policies), here we consider policies which also determine \emph{when} to notify donors, subject to a rate-limiting constraint.
As discussed in Section~\ref{sec:matching-policies} it is necessary to limit the frequency that donors receive notifications; here, we require that donors are notified \emph{at most} once every $K$ days.
As in the previous section, we first describe the offline-optimal policy for a known demand realization $\hat p_{vt}$; this policy is identified using an optimal solution to Problem~\ref{eq:ratelimit-offline-opt}.
\begin{equation}
\label{eq:ratelimit-offline-opt}
    \begin{array}{rll}
        \max & \sum\limits_{t\in \mT}\sum\limits_{e\in E} w_{et} x_{et}\\
        \text{s.t.} & x_{et} \in \{0,1\} & \forall e\in E\; t\in \mT \\
        & a_{ut} \in \{0,1\} & \forall u\in U\; t\in \mT \\
        & s_{v} \in \mathbb{R} &\forall v\in V\\
        & x_{et} \leq \hat p_{vt} & \forall e=(u,v) \in E, \,t\in \mT\\
            & x_{et} \leq a_{ut} & \forall e=(u,v) \in E, \,t\in \mT\\
        & \sum\limits_{e\in E_{u:}} x_{et} \leq a_{ut} & \forall u\in U,\, t\in \mT \\
        & a_{ut} = 1 - \sum\limits_{t'=t-K+1}^{t-1}\sum\limits_{e\in E_{u:}^t} x_{et} & \forall u\in U,\, t\in \mT \\
        & s_{v} = \frac{1}{m_{v}} \sum\limits_{t\in \mT}\sum\limits_{e\in E_{:v}} x_{et} w_{et} & \forall v\in V\\
        & \gamma s_{v} \leq s_{v'} & \forall v,v'\in V,\, v\neq v'.\\
    \end{array}
\end{equation}

This problem differs from the fixed-time setting (Problem~\ref{eq:fixedtime-offline-opt}) in that donor availability $a_{ut}$ is not pre-determined, rather it depends on past matching decisions: on time $t$, if donor $u$ has been matched in the prior $K-1$ time steps, then $a_{ut}=1$, and otherwise $a_{ut}=0$; thus, $a_{ut}\in \{0, 1\}$ is an auxiliary variable defined using constraint $a_{ut} = 1 - \sum\limits_{t'=t-K+1}^{t-1}\sum\limits_{e\in E_{u:}^t} x_{et}$.
Using an optimal solution to Problem~\ref{eq:ratelimit-offline-opt}, offline optimal policy \texttt{OPT}($\gamma$) and competitive ratio $CR$ are defined identically here as in the fixed-time setting.

Further, both baseline policies \RAND{} and \MAX{}, as well as expected \fairness{} metric $EP$ are defined identically here as in the fixed-time setting; however, in the rate-limited setting donors are \emph{available} only if they have not been matched in any of the previous $K-1$ time steps.
As before, \RAND{} is $1$-\fair{} by definition, while \MAX{} is still $0$-\fair{} in the worst case (using the same example as in Lemma~\ref{lem:fixedtime-max-0-fair}).

However, unlike in the fixed-time setting, \MAX{} does not always maximize competitive ratio.
This is intuitive: policies \RAND{} and \MAX{} are myopic, in the sense that they ignore changes in edge weights or donor availability over time.
Instead they match donors as soon as they are available (once every $K$ days at most, if there is an available edge), which can lead to a matching with arbitrarily low weight.
Consider an example donation graph with one donor and one recipient, with two time steps and $K=2$ (the donor may be matched once). 
For $t=1$ the edge weight is $\epsilon\simeq 0$, while for $t=2$ the edge weight is $1$. 
Since both \MAX{} and \RAND{} both match the donor on the first time step $t=1$, the competitive ratio $CR$ can be arbitrarily small.
\begin{lemma}
In the rate-limited setting, the competitive ratio for both \MAX{} and \RAND{} is $CR=\epsilon$, where $\epsilon$ is the smallest edge weight in the graph.
\end{lemma}

As in the previous section, we investigate stochastic non-myopic policies. 
Mirroring our analysis of the fixed-time setting, we first investigate non-adaptive policies, and we extend these to develop approximate adaptive policies.

\paragraph{Non-Adaptive Rate-Limited Policies}
The policies in this section are analogous to the non-adaptive fixed-time policies, but for a rate-limited setting. 
Surprisingly, the guarantees on competitive ratio and expected \fairness{} for these policies are the identical to those in the fixed-time setting.

We begin with a policy based on the an LP relaxation or Problem~\ref{eq:ratelimit-offline-opt}, which refer to as Problem~\ref{eq:ratelimit-offline-opt}-LP. 
As before, this relaxation is almost identical to Problem~\ref{eq:ratelimit-offline-opt}; the only difference being that variables $x_{et}$ and $a_{et}$ are continuous on $[0, 1]$ rather than binary.
As before, this problem yields a valid upper bound on the objective of Problem~\ref{eq:ratelimit-offline-opt}.
\begin{lemma}
Let $Z_{LP}$ denote the optimal objective of Problem~\ref{eq:ratelimit-offline-opt}-LP for matching problem $\mathcal P=(U,V,E,m_v,p_{vt},\mT)$ and $\gamma \in [0, 1]$. 
Let $E[\OPT{\gamma}]$ be the expected objective of the offline-optimal policy, over all demand realizations.
Then, $ Z_{LP} \geq E[\OPT{\gamma}]$.
\end{lemma}
The proof of this lemma is nearly identical to that of Lemma~\ref{lem:lp-ub}, and we omit it here.

The first non-adaptive policy for the rate-limited setting is based on an optimal solution to Problem~\ref{eq:ratelimit-offline-opt}-LP, and is analagous to \texttt{NadapLP} from the previous section:

\begin{definition}[\texttt{NAdapLP\_\,Rate}($\alpha,\gamma$)]
Let $x^*_{et}$ denote an optimal solution to Problem~\ref{eq:ratelimit-offline-opt}-LP, with \fairness{} parameter $\gamma$.
For each time step $t\in \mT$ and each donor $u\in U$, edge $e\in E_{u:}$ is pre-matched with probability $\alpha x^*_{et}/\beta_{ut} p_{vt}$, and the donor is not pre-matched with probability $1 - \alpha \sum_{e=(u,v) \in E_{u:}} \frac{x^*_{et}}{\beta_{ut} p_{vt}}$.
Each parameter $\beta_{ut}$ is equal to the probability that donor $u$ is available at time $t$ under this policy; these parameters are estimated via simulation.\footnote{Please see \cite{Dickerson18:Allocation} for a discussion of this method, which inspired this policy.}
At each time step, all donors with a pre-matched edge for the time step are matched---if both the donor and recipient are available.
\end{definition}

Somewhat surprisingly, each of the important properties of \texttt{NAdapLP} also apply to \texttt{NAdapLP\_Rate}; the proofs are nearly equivalent to the corresponding proofs in the fixed-time setting, and we omit them here.
\begin{lemma}
 Let $x^*_{et}$ be the optimal solution used in policy \texttt{NAdapLP\_Rate}($\alpha,\gamma$).
 The unconditional probability that edge $e$ is matched at time $t$ by policy \texttt{NAdapLP\_Rate} is $\alpha x^*_{et}$.
\end{lemma}
\begin{corollary}\label{cor:nadap-rate-cr}
 \texttt{NAdapLP\_Rate}($\alpha,\gamma$) achieves competitive ratio $CR=\alpha$. 
\end{corollary}
\begin{corollary}
 \texttt{NAdapLP\_Rate}($\alpha,\gamma$) is always $\gamma$-\fair{} in expectation.
\end{corollary}

As in the fixed-time setting, policy  \texttt{NAdapL\_Rate}($\alpha,\gamma$) can only be implemented if $\alpha$ is small enough that the policy is valid.
\begin{lemma}
Policy \texttt{NAdapLP\_Rate}($1/(2D),\gamma$) is always valid and achieves a competitive ratio of $CR \geq 1/(2D)$ for all $\gamma\in [0, 1]$, where $D$ is the maximum degree of any donor: $D \equiv \max_{u\in U} |E_{u:}|$.
\end{lemma}
\begin{proof}
First we observe that \texttt{NAdapLP\_Rate}($\alpha,\gamma$) is valid if $\alpha \leq \beta_{ut} / D$, where $D$.
Next, we show that $\beta_{ut}\geq 1/2$ for policy \texttt{NAdapLP\_Rate}($1/(2D),\gamma$); thus we set $\alpha \gets 1/(2D)$ for the remainder of this proof.
To demonstrate this, we assume that all donors are available at the first time step ($\beta_{u1}=1$), and thus $\beta_{u1}\geq 1/2$.
For all other time steps, $\beta_{ut}$ is expressible as
$$\beta_{ut} \equiv 1 - \sum\limits_{t'=t - K + 1}^{t - 1} P_{ut}$$
where $X_{et}$ is the probability that $u$ is matched at time $t$.
Thus, we can express $\beta_{ut}$ in terms of the decision variables $x^*_{et}$ used to define policy \texttt{NAdapLP\_Rate}($\alpha,\gamma$):
\begin{align*}
    \beta_{ut} &= 1 - \sum\limits_{t'=t - K + 1}^{t - 1} \sum_{e \in E_{:u}} \alpha x^*_{et} \\
    &\geq 1 - \frac{\alpha}{D}\\
    &= 1/2
\end{align*}
Thus, for $\alpha = 1/(2D)$, $\beta_{ut} \geq 1/2$, and $\alpha \leq \beta_{ut}/D$.
Therefore policy \texttt{NAdapLP\_Rate}($1/(2D)$) is always valid; due to Corollary~\ref{cor:nadap-rate-cr} this policy achieves competitive ratio $CR=1/(2D)$.
\end{proof}

\end{document}